\documentclass[preprint,12pt,numbers,sort&compress]{elsarticle}

\usepackage{amsmath,amssymb,amsfonts,mathtools}
\usepackage{amsthm}
\usepackage{bm}
\newtheorem{proposition}{Proposition}

\usepackage{graphicx}
\usepackage{booktabs}
\usepackage{multirow}
\usepackage{array}
\usepackage{algorithm}
\usepackage{subcaption}
\usepackage{algpseudocode}
\usepackage{float}
\usepackage{pgfplots}
\pgfplotsset{compat=1.18}
\usepgfplotslibrary{groupplots}

\usepackage[hidelinks]{hyperref}
\usepackage[nameinlink,noabbrev]{cleveref}
\usepackage{url}

\journal{Journal of Computational Physics}

\begin{document}

\begin{frontmatter}

\title{Cellular Sheaf Neural Operators for Structure-Preserving Surrogate Modeling of Constrained PDEs}

\author[gt,fitmath]{Lennon J. Shikhman\corref{cor1}}
\ead{lshikhman3@gatech.edu}

\author[fitmath,fitaero]{Shane Gilbertie}
\ead{sgilbertie2025@my.fit.edu}

\cortext[cor1]{Corresponding author}

\affiliation[gt]{organization={College of Computing, Georgia Institute of Technology},
                city={Atlanta},
                postcode={30332},
                state={GA},
                country={USA}}

\affiliation[fitmath]{organization={Department of Mathematics and Systems Engineering, Florida Institute of Technology},
                city={Melbourne},
                postcode={32901},
                state={FL},
                country={USA}}

\affiliation[fitaero]{organization={Department of Aerospace, Physics and Space Sciences, Florida Institute of Technology},
                city={Melbourne},
                postcode={32901},
                state={FL},
                country={USA}}

\begin{abstract}
Neural operators provide fast surrogate models for PDE simulations, but standard architectures often treat geometry and discretization as secondary to field data. Physical states are usually represented as grid-channel stacks, even when different quantities naturally belong on vertices, edges, faces, cells, boundaries, or interfaces and must satisfy compatibility constraints. We propose \emph{Cellular Sheaf Neural Operators}, a discretization-aware framework for structure-preserving neural PDE surrogates. The method represents PDE states on oriented cell complexes, couples local feature spaces through learned restriction maps, and uses incidence/Hodge-informed message passing to follow computational geometry. Learned update heads pass through coboundary or flux maps, allowing selected constraints to arise from cell-complex structure rather than only from loss penalties. For magnetohydrodynamics, this yields face-based magnetic-flux updates driven by edge electromotive fields and finite-volume-style fluid updates driven by learned face fluxes and cell sources. On turbulent MHD and fusion-equilibrium surrogate tasks, the method improves structure-sensitive diagnostics, including rollout behavior, divergence control, spectral error, and equilibrium-regression accuracy. These results indicate that cellular-sheaf structure is a useful inductive bias for neural PDE surrogates in constrained multiphysics systems.
\end{abstract}

\begin{keyword}
neural operators \sep cellular sheaves \sep PDE surrogate modeling \sep structure-preserving learning \sep magnetohydrodynamics \sep scientific machine learning
\end{keyword}

\end{frontmatter}

\section{Introduction}
\label{sec:introduction}

Neural operators have become a standard approach for learning fast surrogate models of partial differential equation (PDE) solution maps and time-advance operators \citep{kovachki2023neural,li2021fourier,lu2021learning}. By amortizing the cost of expensive simulations, these models can accelerate parameter studies, uncertainty quantification, design optimization, and inverse problems. Most neural PDE surrogates, however, represent the simulation state as a homogeneous tensor of grid or mesh channels, a convenient but inaccurate representation for neural architectures which does not reflect the typed geometric structure of many numerical discretizations. 

In many PDE solvers, physical quantities are not just interchangeable channels: conserved quantities may be stored in cells, fluxes on faces, circulations on edges, and potentials or boundary quantities on vertices. Structure-preserving numerical methods exploit this placement through compatible discretizations, finite-volume updates, incidence matrices, Hodge operators, and identities such as the boundary of a boundary being zero \citep{hyman1997adjoint,hyman1999orthogonal,desbrun2005dec,arnold2006feec,arnold2010feec}. From a machine-learning perspective, this creates an architectural mismatch, in which a neural operator may be trained on fields whose variables have different geometric roles, while the model class treats them as a single untyped array. This mismatch can be hidden in one-step prediction error but is critical in autoregressive rollout where small compatibility violations can accumulate quickly into physically inconsistent states \citep{shikhman2026diagnosing}.

This paper introduces \emph{Cellular Sheaf Neural Operators} (CSNO), a cochain-valued neural operator framework for structure-preserving, compatibility-aware PDE surrogate models. CSNO replaces the homogeneous channel representation with fields attached to the vertices, edges, faces, and cells of an oriented cell complex. Learned cellular-sheaf restriction maps couple local feature spaces across incidences, while incidence/Hodge-informed message passing gives the architecture access to the same algebraic structure used by compatible numerical methods. The central design choice is that selected update heads do not directly predict arbitrary next-state channels. Instead, they learn fluxes, circulations, source terms, or potentials and route them through coboundary or flux maps. This makes the learned operator class itself compatibility-aware: selected constraints are enforced by cell-complex identities such as \(d_{k+1}d_k=0\), rather than only encouraged by loss penalties. Magnetohydrodynamics (MHD) provides a useful testbed because it couples fluid variables with a constrained magnetic field. The magnetic field must satisfy
\begin{equation}
    \nabla \cdot B = 0.
\end{equation}
Numerical magnetic divergence can introduce nonphysical effects and damage stability \citep{brackbill1980effect,toth2000divb}. In CSNO, the MHD update follows the constrained-transport idea at the level of the learned operator parameterization. Magnetic flux is represented as a face cochain, the network predicts an edge-based electromotive field, and the magnetic update passes through the edge-to-face coboundary. The identity \(d_2d_1=0\) then gives a divergence-preserving magnetic update in the native cochain representation. Fluid variables are advanced through a related finite-volume-style update using learned face fluxes and cell source terms.

We evaluate CSNO on turbulent MHD from The Well MHD\_64 \citep{ohana2024well} and on fusion-relevant ideal-MHD equilibrium prediction from ConStellaration \citep{cadena2025constellaration}. The Well MHD\_64 evaluates CSNO as a learned time-stepper on three-dimensional turbulent fields, while ConStellaration evaluates whether grouped sheaf-style representations are useful for structured equilibrium regression. The experiments compare ordinary one-step prediction accuracy with structure-sensitive behavior, including magnetic-divergence diagnostics, rollout stability, spectral error, parameter efficiency, and equilibrium-regression accuracy. This evaluation reflects the main claim of the paper: learned PDE surrogates should be assessed not only as one-step tensor regressors, but also as autoregressive operators whose inductive biases affect long-horizon and compatibility-sensitive behavior.

\subsection{Contributions}
\label{subsec:contributions}

The main contributions are:
\begin{enumerate}
    \item We introduce \emph{Cellular Sheaf Neural Operators}, a cochain-valued neural operator framework that represents PDE states on oriented cell complexes rather than as homogeneous grid-channel tensors.

    \item We adapt cellular-sheaf message passing to PDE time-advance modeling by coupling learned restriction maps with incidence/Hodge-informed communication across cells of different dimensions.

    \item We define compatibility-preserving update heads that factor learned quantities through coboundary or flux maps, yielding a constrained operator class whose selected compatibility residuals are preserved by identities such as \(d_{k+1}d_k=0\).

    \item We instantiate CSNO for MHD with face-based magnetic flux, edge-based electromotive updates, and cell-based fluid variables, and evaluate it on turbulent MHD and fusion-equilibrium surrogate tasks.
\end{enumerate}

\section{Related work}
\label{sec:related_work}

\subsection{Structure-preserving computational methods for constrained PDEs}
\label{subsec:structure_preserving_methods}

Structure preservation is a central principle in numerical methods for constrained PDEs, especially when rollout stability depends on respecting conservation laws, compatibility conditions, or geometric identities. For conservation laws and convection-dominated systems, finite-volume and discontinuous Galerkin methods encode local conservation, numerical fluxes, and stable propagation of discontinuities directly into the discrete update \citep{leveque2002finite,cockburn2001runge,hesthaven2008nodal}. Compatible and mimetic discretizations pursue a related goal for geometric constraints by constructing discrete gradient, curl, and divergence operators that preserve vector-calculus identities and adjoint relationships at the algebraic level \citep{hyman1997adjoint,hyman1999orthogonal,hyman2000james,brezzi2005convergence}.

Discrete exterior calculus and finite element exterior calculus provide a broader language for these ideas by representing fields as cochains or discrete differential forms on vertices, edges, faces, and cells, while preserving de Rham complex structure, coboundary maps, Hodge operators, and homological identities in finite-dimensional spaces \citep{desbrun2005dec,hirani2003dec,arnold2006feec,arnold2010feec}. This perspective is especially natural in computational electromagnetism, where edge and face degrees of freedom reflect the physical roles of circulation, flux, and charge \citep{bossavit1998computational}. The central lesson is that stable computation for constrained PDEs often requires preserving the algebraic structure of the continuous problem, not only reducing loss alone.

Magnetohydrodynamics is a clear example. The magnetic field must satisfy $\nabla \cdot B = 0$, and numerical violations can introduce nonphysical forces and degrade stability. Classical MHD solvers therefore use constrained transport, divergence-free adaptive mesh refinement, projection-type corrections, and hyperbolic or parabolic divergence cleaning to preserve or control magnetic divergence \citep{evans1988constrained,balsara2001amr,toth2000divb,dedner2002hyperbolic}. CSNO brings the structure-preserving tradition to neural PDE surrogate modeling. Rather than predicting all variables as generic grid channels, the proposed method represents fields on cell-complex entities and uses discrete incidence structure to build compatibility into the learned update.

\subsection{Neural operators and surrogate modeling for scientific computing}
\label{subsec:neural_operators_surrogates}

Neural operators learn mappings between function spaces and are widely used as surrogate models for PDE solution operators or time-advance maps \citep{kovachki2023neural,lu2021learning,li2021fourier}. DeepONet realizes operator learning through branch and trunk networks and has been demonstrated on deterministic and stochastic differential-equation operators \citep{lu2021learning}. Fourier Neural Operators parameterize global convolution kernels in Fourier space and have shown strong performance on grid-based PDE benchmarks, including turbulent-flow surrogate modeling and zero-shot super-resolution \citep{li2021fourier}. Graph-based neural operators extend this viewpoint to nonuniform discretizations through message-passing approximations of kernel integration, while mesh-based graph simulators provide a related approach for learned time-stepping on adaptive simulation meshes \citep{li2020neural,pfaff2021learning}. Geometry-aware variants such as Geo-FNO address the limitation that standard FFT-based FNOs are naturally tied to rectangular uniform grids by learning deformations from physical domains to latent regular grids \citep{li2023fourier}.

These models are attractive in scientific computing because a trained surrogate can amortize the cost of repeated simulations for parameter studies, uncertainty quantification, control, and inverse problems. Physics-informed neural operators further incorporate PDE residual information into operator learning, combining data-driven supervision with physics constraints \citep{li2024physics}. Large benchmark suites such as PDEBench and The Well have standardized evaluation across time-dependent PDE datasets and commonly include U-Net, FNO, and related baselines \citep{takamoto2022pdebench,ohana2024well}.

However, many neural surrogates still operate on a homogeneous channel representation of the physical state. This is convenient for grids, but it can obscure the computational placement of physical variables. A scalar conserved quantity, a face flux, an edge circulation, and a boundary value may all appear as channels even though they play different roles in the discrete PDE update. Boundary data create a related difficulty: changing boundary conditions can induce distinct boundary-indexed operator families, so a single boundary-agnostic surrogate may fail under boundary-condition shift \citep{shikhman2026one}. Together, these issues motivate neural surrogate models that operate not only on functions or grid tensors, but also on algebraic structures used by numerical methods, including incidence matrices, cochains, Hodge maps, and compatibility-preserving updates.

\subsection{Sheaf, geometric, and topological deep learning}
\label{subsec:sheaf_geometric_learning}

Geometric deep learning designs neural architectures whose inductive biases reflect the structure of the underlying domain, including grids, graphs, manifolds, and meshes \citep{bronstein2021geometric}. Graph convolutional and message-passing networks provide the standard framework for learning on relational data by aggregating information over local neighborhoods \citep{kipf2017semi,gilmer2017neural}, and graph networks have been successfully used as learned physics simulators on interacting systems and meshes \citep{sanchez2018graph,pfaff2021learning}. However, ordinary graph neural networks usually attach all learned quantities to nodes and model pairwise relations. This is limiting for PDE systems in which variables naturally live on cells of different dimensions.

Topological deep learning extends neural architectures from graphs to richer domains such as simplicial complexes, cellular complexes, hypergraphs, and sheaves \citep{hajij2023topological,papillon2024architectures,papamarkou2024topological}. Simplicial and cellular neural networks use incidence relations between vertices, edges, faces, and cells to define higher-order message passing \citep{ebli2020simplicial,bodnar2021weisfeiler,hajij2020cell_complex,bodnar2021cellular}. Hodge-based methods use cochain structure and Hodge Laplacians to process edge- or higher-order signals, separating gradient-like, curl-like, and harmonic components \citep{roddenberry2019hodgenet,roddenberry2021principled,schaub2021higher_order_signal_processing}. These methods are closely aligned with numerical PDE structure because they operate on the same kind of cellwise and incidence-based information used in compatible discretizations.

Cellular sheaves add another layer of structure by assigning local vector spaces, or stalks, to cells and using restriction maps to express how neighboring local data should agree. The associated sheaf Laplacian measures inconsistency under these restrictions, generalizing graph Laplacians to heterogeneous local feature spaces \citep{hansen2019toward}. Sheaf Neural Networks, Neural Sheaf Diffusion, and connection-Laplacian models use this structure to replace ordinary graph diffusion with diffusion over learned stalks, restriction maps, and aligned local spaces \citep{hansen2020sheaf,bodnar2022neural_sheaf_diffusion,barbero2022sheaf_connection}.

The proposed Cellular Sheaf Neural Operator builds on this viewpoint for PDE surrogate modeling. Instead of using sheaf structure only as a graph-learning mechanism, we use it to organize cochain-valued physical variables on a cell complex. Learned restriction maps couple local feature spaces, while incidence and Hodge structure provide the algebraic backbone for message passing and compatibility-preserving updates. The resulting architecture uses sheaf structure not only for representation learning, but also for organizing learned PDE updates.

\subsection{Computational magnetohydrodynamics and divergence control}
\label{subsec:computational_mhd}

Magnetohydrodynamics models electrically conducting fluids by coupling fluid conservation laws with electromagnetic field evolution. In ideal and compressible MHD, the state typically includes density, momentum or velocity, pressure or total energy, and the magnetic field, whose evolution is governed by the induction equation \citep{goedbloed2004principles,freidberg2014ideal}. This coupling makes MHD a useful test case for structure-preserving surrogate modeling: its variables are strongly coupled, geometrically distinct, and constrained by compatibility conditions. In turbulent MHD, this structure is also reflected in the decomposition of fluctuations into Alfvén, slow magnetosonic, and fast magnetosonic modes, whose energy fractions and anisotropies depend on forcing and plasma parameters \citep{makwana2020properties}.

A central constraint is the solenoidal condition
\begin{equation*}
    \nabla \cdot B = 0,
\end{equation*}
which expresses the absence of magnetic monopoles. Nonzero numerical magnetic divergence can introduce nonphysical forces, corrupt wave propagation, and degrade the computed MHD solution \citep{brackbill1980effect}. Consequently, MHD solvers often treat divergence control as part of the discretization rather than as a secondary diagnostic.

Several numerical strategies preserve or control this constraint. Finite-volume and Godunov-type solvers use conservative updates and approximate Riemann solvers for shocks, discontinuities, and nonlinear wave interactions \citep{powell1999solution,miyoshi2005multi}. Constrained-transport methods evolve magnetic fluxes using staggered magnetic and electromotive quantities so that a discrete divergence constraint is preserved; modern multidimensional MHD codes combine this idea with shock-capturing Godunov updates for stable high-resolution simulation \citep{gardiner2008unsplit,stone2008athena}. Other approaches include divergence cleaning, projection corrections, and compatible discretizations based on discrete gradient, curl, and divergence operators. Recent neural-operator work for MHD similarly suggests that surrogate models benefit from architectures and losses that reflect numerical-flux structure and divergence-related physics \citep{kim2025operatorlearning}. The MHD version of CSNO follows this geometric placement principle. Magnetic flux is represented on faces, electromotive quantities on edges, and the magnetic update is built from the edge-to-face coboundary map. Since the discrete coboundary satisfies \(d_2 d_1 = 0\), updating face magnetic flux through an edge electromotive field preserves discrete magnetic divergence. Thus, MHD provides a demanding test case for cochain placement, incidence structure, and compatibility-preserving learned updates.
\section{Cellular Sheaf Neural Operator}
\label{sec:method}

The Cellular Sheaf Neural Operator (CSNO) is designed for constrained PDE systems whose variables have different geometric and computational roles. Rather than representing the full state as a homogeneous stack of channels, CSNO represents physical quantities as cochain-valued fields on an oriented cell complex. Learned cellular-sheaf restriction maps couple local feature spaces across incidences, while incidence and Hodge structure provide the algebraic backbone for message passing and compatibility-preserving updates.

\subsection{Cell complexes and cochain-valued states}
\label{subsec:cell_complex_cochains}

Let \(K\) be an oriented cell complex with \(k\)-cells \(K_k\), \(k=0,\ldots,d\). In three dimensions,
\begin{equation*}
    K = K_0 \cup K_1 \cup K_2 \cup K_3,
\end{equation*}
where \(K_0\) denotes vertices, \(K_1\) edges, \(K_2\) faces, and \(K_3\) volume cells. The oriented boundary maps are
\begin{equation*}
    \partial_k : C_k(K) \to C_{k-1}(K),
\end{equation*}
and the corresponding coboundary maps are
\begin{equation*}
    d_{k-1} = \partial_k^{\top} : C^{k-1}(K) \to C^k(K).
\end{equation*}
These maps satisfy the exact algebraic identity
\begin{equation}
    d_k d_{k-1} = 0,
    \label{eq:coboundary_identity_general}
\end{equation}
the discrete analogue of vector-calculus identities such as \(\nabla \cdot (\nabla \times A)=0\). This identity is central in discrete exterior calculus, finite element exterior calculus, and compatible discretizations \citep{desbrun2005dec,hirani2003dec,arnold2006feec,arnold2010feec}.

A cochain-valued PDE state assigns features to cells of different dimensions:
\begin{equation}
    u
    =
    \left(
        u^{(0)},u^{(1)},\ldots,u^{(d)}
    \right),
    \qquad
    u^{(k)} \in C^k(K;\mathbb{R}^{c_k}).
    \label{eq:cochain_state}
\end{equation}
Here \(u^{(k)}\) is a feature array on \(k\)-cells, and \(c_k\) is the number of physical or learned channels associated with that cell dimension. This representation allows the model to distinguish quantities by their computational role. Vertex data, edge circulations, face fluxes, cell-centered conserved variables, and boundary or interface quantities need not be collapsed into the same channel space. This cochain-valued viewpoint is consistent with recent work on learning with differential \(k\)-forms and higher-order fields on simplicial or cellular domains \citep{maggs2024simplicial}.

The time-dependent surrogate-learning problem is to approximate a discrete evolution map
\begin{equation*}
    \mathcal{G}_{\Delta t}: u^n \mapsto u^{n+1},
\end{equation*}
or, with \(m\) input states,
\begin{equation*}
    \mathcal{G}_{\Delta t}:
    \left(u^{n-m+1},\ldots,u^n\right)
    \mapsto u^{n+1}.
\end{equation*}
When data are stored as structured tensors, the model first maps them to cochain fields on an associated cell complex, applies the cellular-sheaf update, and then projects the prediction back to the dataset storage format.

\subsection{Cellular sheaves and learned restriction maps}
\label{subsec:cellular_sheaf_restrictions}

A cellular sheaf assigns a vector space, or stalk, to each cell of a cell complex, together with restriction maps that compare data across incident cells \citep{hansen2019toward,hansen2020sheaf}. For a cell \(\sigma \in K\), let
\begin{equation*}
    \mathcal{F}(\sigma) \cong \mathbb{R}^{r_{\dim \sigma}}
\end{equation*}
denote the learned feature space attached to \(\sigma\). If \(\tau \leq \sigma\), meaning \(\tau\) is a face of \(\sigma\), a restriction map
\begin{equation*}
    \rho_{\sigma \to \tau}:
    \mathcal{F}(\sigma) \to \mathcal{F}(\tau)
\end{equation*}
describes how information on \(\sigma\) is compared or communicated to \(\tau\).

In CSNO, these restriction maps are trainable. For an incidence \(\tau \leq \sigma\), we use
\begin{equation*}
    \rho_{\sigma \to \tau}^{\theta}
    =
    R_{\dim\sigma,\dim\tau}^{\theta},
\end{equation*}
with optional dependence on local geometry:
\begin{equation*}
    \rho_{\sigma \to \tau}^{\theta}
    =
    R_{\dim\sigma,\dim\tau}^{\theta}
    \bigl(g_{\sigma},g_{\tau},g_{\sigma\tau}\bigr).
\end{equation*}
The geometric features \(g_{\sigma}\), \(g_{\tau}\), and \(g_{\sigma\tau}\) may include cell volumes, face areas, edge lengths, centroids, normals, orientations, or boundary tags. Shared restriction maps tie parameters across incidences of the same type, while geometry-conditioned maps allow the coupling to adapt to local mesh structure.

The learned restrictions do not impose a fixed PDE law. They define a structured communication pattern over the cell complex. This is the central distinction between CSNO and a standard channel-mixing neural operator. Information is passed through the incidence geometry of the computational domain rather than through a single homogeneous feature tensor.

\subsection{Sheaf coboundary and Hodge message passing}
\label{subsec:sheaf_coboundary_hodge}

The learned restriction maps define sheaf coboundary operators
\begin{equation}
    \delta_{\mathcal{F},k}:
    C^k(K;\mathcal{F}) \to C^{k+1}(K;\mathcal{F}),
    \label{eq:sheaf_coboundary}
\end{equation}
which generalize ordinary coboundary matrices to heterogeneous learned feature spaces. A corresponding sheaf-Hodge Laplacian on \(k\)-cochains is
\begin{equation}
    L_{\mathcal{F},k}
    =
    \delta_{\mathcal{F},k}^{\top}
    \delta_{\mathcal{F},k}
    +
    \delta_{\mathcal{F},k-1}
    \delta_{\mathcal{F},k-1}^{\top}.
    \label{eq:sheaf_hodge_laplacian}
\end{equation}
This operator measures inconsistency under the learned restrictions and generalizes graph Laplacian diffusion to cellular sheaves \citep{hansen2019toward,bodnar2022neural_sheaf_diffusion,barbero2022sheaf_connection}.

A cellular CSNO block updates hidden cochains
\begin{equation*}
    h_k^\ell \in C^k(K;\mathbb{R}^{r_k}),
    \qquad k=0,\ldots,d,
\end{equation*}
using within-dimension updates, upward incidence messages, downward incidence messages, and optional sheaf-Hodge diffusion:
\begin{equation}
    h_k^{\ell+1}
    =
    \sigma
    \left(
        A_k h_k^\ell
        +
        B_k L_{\mathcal{F},k} h_k^\ell
        +
        C_k \delta_{\mathcal{F},k-1} h_{k-1}^\ell
        +
        D_k \delta_{\mathcal{F},k}^{\top} h_{k+1}^\ell
    \right).
    \label{eq:csno_block}
\end{equation}
Here \(A_k,B_k,C_k,D_k\) are learned channel maps, \(\sigma\) is a nonlinear activation, and terms involving nonexistent dimensions are omitted. Incidence message passing is the default update mechanism; sheaf-Hodge Laplacian filtering can be enabled as an additional diffusion-like component. This construction is related to higher-order message-passing architectures on cellular complexes, including equivariant cellular networks, but here the cell-complex message passing is used inside a neural PDE time-advance surrogate rather than only as a representation-learning layer \citep{kovac2024empcn}.

Geometric Hodge weights enter through diagonal mass-like operators \(M_k\), defining weighted inner products
\begin{equation*}
    \langle a,b\rangle_{k,h}
    =
    a^{\top} M_k b.
\end{equation*}
For cubical complexes, these weights are derived from edge lengths, face areas, and cell volumes. For general cell complexes, the same interface supports geometry-dependent diagonal Hodge weights supplied with the mesh.

\subsection{General cochain update heads}
\label{subsec:general_update_heads}

After \(L\) cellular-sheaf blocks, the model has hidden cochains
\begin{equation*}
    h^L =
    \left(
        h_0^L,\ldots,h_d^L
    \right).
\end{equation*}
CSNO does not require every output to be predicted on the same cell dimension. Task-specific output heads produce cochain-valued quantities on the dimensions appropriate to the PDE discretization:
\begin{equation*}
    \widehat{q}^{(k)}
    =
    P_k(h_0^L,\ldots,h_d^L),
    \qquad
    \widehat{q}^{(k)} \in C^k(K;\mathbb{R}^{m_k}).
\end{equation*}
These heads may predict next-state variables directly, residuals, fluxes, sources, potentials, or constraint-carrying auxiliary quantities.

A general residual update can be written as
\begin{equation*}
    \widehat{u}^{(k),n+1}
    =
    u^{(k),n}
    +
    \Delta t\,\mathcal{U}^{(k)}_{\theta}(h^L,K),
\end{equation*}
where \(\mathcal{U}^{(k)}_{\theta}\) is a learned cochain update. For structure-preserving tasks, these heads may instead produce auxiliary cochains that are routed through coboundary or flux maps. The resulting factorized update and its compatibility guarantee are stated in \cref{subsec:structural_guarantees}.

Flux-form updates are another important case. For cell-centered quantities \(U \in C^d(K)\), a finite-volume-style update can be written
\begin{equation}
    \widehat{U}^{n+1}
    =
    U^n
    -
    \Delta t\, d_{d-1}F_\theta
    +
    \Delta t\,S_\theta,
    \label{eq:general_flux_update}
\end{equation}
where \(F_\theta \in C^{d-1}(K)\) is a learned flux cochain and \(S_\theta \in C^d(K)\) is a learned source or residual. This form separates flux-like information from cell-centered state information and aligns the neural update with conservative discretization structure.

\subsection{Structural guarantees}
\label{subsec:structural_guarantees}

The identity \(d_{k+1}d_k=0\) is classical; the architectural choice in CSNO is to make selected learned update heads factor through the coboundary \(d_k\). This turns a standard algebraic property of the cell complex into a constraint on the learned time-stepper. Let \(K\) be an oriented cell complex with coboundary maps
\begin{equation*}
    C^k(K) \xrightarrow{d_k} C^{k+1}(K) \xrightarrow{d_{k+1}} C^{k+2}(K),
    \qquad d_{k+1}d_k=0.
\end{equation*}
For a constrained \((k+1)\)-cochain, CSNO parameterizes the update as
\begin{equation}
    \widehat{u}^{(k+1),n+1}
    =
    u^{(k+1),n}
    +
    d_k a_\theta^{(k)} ,
    \label{eq:factorized_update_general}
\end{equation}
where \(a_\theta^{(k)}\in C^k(K)\) is produced by the cellular-sheaf network from the hidden cochains \(h^L\). Thus, the network does not directly predict an arbitrary residual in \(C^{k+1}(K)\). Instead, its residual is constrained to lie in
\begin{equation*}
    \mathrm{im}(d_k)\subseteq C^{k+1}(K),
\end{equation*}
and the learned update lies in the affine space
\begin{equation*}
    u^{(k+1),n}+\mathrm{im}(d_k).
\end{equation*}
Equivalently, the constrained CSNO update head has the operator form
\begin{equation}
    \mathcal{G}_\theta
    =
    I + d_k\circ \Phi_\theta ,
    \label{eq:csno_factorized_operator}
\end{equation}
where \(\Phi_\theta\) is the learned cellular-sheaf map that produces the generating cochain. This is the structural difference between a CSNO compatibility-preserving head and a direct channel-wise residual predictor.

\begin{proposition}[Compatibility preservation of the CSNO update head]
\label{prop:compatibility_preservation}
Let \(u^{(k+1),n}\in C^{k+1}(K)\) be updated by the CSNO factorized head in \eqref{eq:factorized_update_general}. Then
\begin{equation*}
    d_{k+1}\widehat{u}^{(k+1),n+1}
    =
    d_{k+1}u^{(k+1),n}.
\end{equation*}
Consequently, if \(d_{k+1}u^{(k+1),0}=0\), then \(d_{k+1}u^{(k+1),n}=0\) for every rollout step generated by this update head.
\end{proposition}

\begin{proof}
Applying \(d_{k+1}\) to the CSNO update in \eqref{eq:factorized_update_general} gives
\begin{equation*}
    d_{k+1}\widehat{u}^{(k+1),n+1}
    =
    d_{k+1}u^{(k+1),n}
    +
    d_{k+1}d_k a_\theta^{(k)}.
\end{equation*}
The second term vanishes because \(d_{k+1}d_k=0\). Therefore
\begin{equation*}
    d_{k+1}\widehat{u}^{(k+1),n+1}
    =
    d_{k+1}u^{(k+1),n}.
\end{equation*}
Repeating the same argument at each rollout step gives the final statement.
\end{proof}

The guarantee is independent of the particular neural-network weights, optimizer, or training error. It follows from the CSNO output parameterization itself. A direct residual head of the form
\begin{equation*}
    \widehat{u}^{(k+1),n+1}
    =
    u^{(k+1),n}
    +
    r_\theta^{(k+1)}
\end{equation*}
has the same preservation property only if
\begin{equation*}
    r_\theta^{(k+1)}\in \ker(d_{k+1})
\end{equation*}
for every input. A standard U-Net, FNO, or unconstrained channel-wise neural operator does not enforce this condition. CSNO enforces it by constructing the residual as \(d_k a_\theta^{(k)}\), which automatically belongs to \(\ker(d_{k+1})\) by the cochain-complex identity.

For the MHD specialization, the constrained head is applied with \(k=1\). CSNO predicts an edge cochain
\begin{equation*}
    E_\theta^n\in C^1(K;\mathbb{R}^{c_E})
\end{equation*}
and updates the face-based magnetic flux by
\begin{equation}
    \widehat{B}^{n+1}
    =
    B^n-\Delta t\,d_1E_\theta^n.
    \label{eq:csno_mhd_factorized_update}
\end{equation}
The discrete magnetic divergence is \(d_2B\). Applying Proposition~\ref{prop:compatibility_preservation} gives
\begin{equation}
    d_2\widehat{B}^{n+1}
    =
    d_2B^n.
    \label{eq:csno_native_divergence_preservation}
\end{equation}
Thus, the CSNO magnetic update cannot introduce native cochain-divergence drift. This is the learned-surrogate analogue of the constrained-transport principle: the neural network learns the edge-based electromotive quantity, but the cell complex determines how that learned quantity changes face magnetic flux.

This native guarantee is distinct from the projected grid-space divergence diagnostic used for cross-model evaluation. The datasets store fields as grid tensors, so CSNO uses projection maps
\begin{equation*}
    \Pi_{\mathrm{data}\to K}
    \qquad\text{and}\qquad
    \Pi_{K\to \mathrm{data}}.
\end{equation*}
Let \(D_h\) denote the finite-difference divergence diagnostic applied after projecting a predicted cochain field back to the tensor representation. Even when \eqref{eq:csno_native_divergence_preservation} holds exactly, the reported grid diagnostic measures
\begin{equation*}
    D_h\Pi_{K\to \mathrm{data}}(\widehat{B})
\end{equation*}
rather than \(d_2\widehat{B}\). The discrepancy is controlled by the noncommutation between the projection and the diagnostic:
\begin{equation}
\begin{aligned}
    \left\|
    D_h\Pi_{K\to \mathrm{data}}(\widehat{B})
    -
    \Pi_{K\to \mathrm{data}}(d_2\widehat{B})
    \right\|
    &\leq
    \left\|
    D_h\Pi_{K\to \mathrm{data}}
    -
    \Pi_{K\to \mathrm{data}}d_2
    \right\|
    \,
    \|\widehat{B}\|.
\end{aligned}
\label{eq:projection_commutator_bound}
\end{equation}
The operator
\begin{equation*}
    D_h\Pi_{K\to \mathrm{data}}
    -
    \Pi_{K\to \mathrm{data}}d_2
\end{equation*}
is the projection-commutator error. It vanishes only when the projection and grid diagnostic commute with the cochain divergence. In practice, interpolation from cell-centered magnetic channels to face cochains, projection back to stored tensor channels, and finite-difference boundary treatment can make the projected divergence nonzero even when the native cochain divergence is preserved.

The same factorization clarifies the rollout behavior measured in the experiments. A direct predictor can accumulate error both in the physical state and in the compatibility residual \(d_{k+1}u\). For the constrained CSNO head, the selected compatibility residual is invariant under rollout:
\begin{equation}
    d_{k+1}\mathcal{G}_\theta^N(u)
    =
    d_{k+1}u
    \qquad
    \text{for all } N\geq 1.
    \label{eq:csno_rollout_residual_invariance}
\end{equation}
The model can still accumulate ordinary prediction error through an imperfect learned generating cochain \(a_\theta^{(k)}\), but it cannot accumulate error through drift in the preserved compatibility quantity. This explains why CSNO can differ from direct tensor predictors in long-horizon diagnostics even when one-step pointwise error is not minimized.

\subsection{MHD specialization}
\label{subsec:mhd_specialization}

The placement of variables on the cell complex is problem-dependent. For the MHD experiments, we use the placement suggested by compatible discretizations and constrained-transport methods. Fluid variables are associated with volume cells,
\begin{equation*}
    U^n \in C^3(K;\mathbb{R}^{c_U}),
\end{equation*}
magnetic flux is associated with faces,
\begin{equation*}
    B^n \in C^2(K;\mathbb{R}^{c_B}),
\end{equation*}
and electromotive quantities are associated with edges,
\begin{equation*}
    E_\theta^n \in C^1(K;\mathbb{R}^{c_E}).
\end{equation*}
This placement mirrors the geometric roles of volume quantities, face fluxes, and edge circulations in compatible MHD methods \citep{bossavit1998computational,evans1988constrained,gardiner2008unsplit,stone2008athena}.

Rather than directly predicting the next magnetic field, the model predicts an edge-based electromotive field and updates face magnetic flux by
\begin{equation}
    \widehat{B}^{n+1}
    =
    B^n
    -
    \Delta t\, d_1 E_{\theta}^n.
    \label{eq:mhd_magnetic_update}
\end{equation}
The discrete magnetic divergence is represented by
\begin{equation*}
    d_2 B,
    \qquad
    d_2 : C^2(K) \to C^3(K).
\end{equation*}
Applying \(d_2\) to \eqref{eq:mhd_magnetic_update} gives
\begin{equation}
    d_2 \widehat{B}^{n+1}
    =
    d_2 B^n
    -
    \Delta t\, d_2 d_1 E_{\theta}^n
    =
    d_2 B^n,
    \label{eq:mhd_divergence_preservation}
\end{equation}
since \(d_2d_1=0\). Thus, if the initial magnetic flux satisfies \(d_2B^0=0\), the learned magnetic update preserves this discrete divergence up to numerical precision.

For the fluid variables, the model predicts face-based fluxes
\begin{equation*}
    F_{\theta}^n \in C^2(K;\mathbb{R}^{c_U})
\end{equation*}
and optional cell-centered source or residual terms
\begin{equation*}
    S_{\theta}^n \in C^3(K;\mathbb{R}^{c_U}).
\end{equation*}
The cell-centered fluid state is updated by
\begin{equation}
    \widehat{U}^{n+1}
    =
    U^n
    -
    \Delta t\, d_2F_{\theta}^n
    +
    \Delta t\,S_{\theta}^n.
    \label{eq:mhd_fluid_update}
\end{equation}
This form mirrors finite-volume structure, where the change in a cell quantity is determined by fluxes across the cell boundary \citep{leveque2002finite}. Unlike the magnetic update, this flux form does not by itself guarantee exact conservation of every physical invariant, because the learned fluxes and source terms may contain modeling error. Its purpose is to align the neural update with the algebraic structure of conservative discretizations.

\subsection{Projection, training, and implementation}
\label{subsec:projection_training_implementation}

Many scientific datasets are stored as arrays rather than native cochain data. To use them with CSNO, we define projection operators
\begin{equation*}
    \Pi_{\mathrm{data}\to K}
    :
    \mathcal{X}_{\mathrm{data}}
    \to
    \bigoplus_{k=0}^d C^k(K;\mathbb{R}^{c_k})
\end{equation*}
and
\begin{equation*}
    \Pi_{K\to \mathrm{data}}
    :
    \bigoplus_{k=0}^d C^k(K;\mathbb{R}^{c_k})
    \to
    \mathcal{X}_{\mathrm{data}}.
\end{equation*}
For the structured-grid datasets used here,
\begin{equation*}
    \mathcal{X}_{\mathrm{data}}
    =
    \mathbb{R}^{C \times N_x \times N_y \times N_z}.
\end{equation*}
A cubical complex is constructed from the preprocessed grid. Cell-centered variables are assigned to volume cells. Magnetic vector components are assigned to oriented faces according to their normal direction; when the raw data provide cell-centered magnetic components, the projection uses a consistent interpolation to face cochains. After the cellular update, predicted cochains are projected back to the channel layout used by the supervised target tensor.

For time-dependent prediction, the primary supervised loss is computed after projecting the cochain prediction back to the dataset representation. For grid-valued outputs, we use mean-squared error and relative \(L^2\) error:
\begin{equation}
    \mathcal{L}_{\mathrm{pred}}
    =
    \frac{1}{|\Omega|C}
    \sum_{c=1}^C
    \sum_{x\in\Omega}
    \left|
        \widehat{u}_c(x)-u_c(x)
    \right|^2
    +
    \lambda_{\mathrm{rel}}
    \frac{\|\widehat{u}-u\|_2}{\|u\|_2}.
    \label{eq:training_loss}
\end{equation}
Additional compatibility diagnostics or regularizers are defined on native cochains. For a constraint represented by a coboundary residual \(d_k q^{(k)}\), we track
\begin{equation*}
    \mathcal{D}_k(q)
    =
    \|d_k q^{(k)}\|_2.
\end{equation*}
When used as a regularizer, the corresponding term is
\begin{equation*}
    \mathcal{L}_{\mathrm{comp}}
    =
    \|d_k \widehat{q}^{(k)}\|_2^2.
\end{equation*}
The total training objective is
\begin{equation*}
    \mathcal{L}
    =
    \mathcal{L}_{\mathrm{pred}}
    +
    \lambda_{\mathrm{comp}}\mathcal{L}_{\mathrm{comp}}.
\end{equation*}
For update heads that preserve a compatibility relation by construction, the corresponding term is primarily a diagnostic or a safeguard against projection error.

For ConStellaration, which is an equilibrium-regression task rather than a time-dependent field-prediction problem, we use a finite-dimensional sheaf-style variant. Input features are partitioned into groups, lifted to learned fibers, coupled through restriction-style maps, and concatenated before the final regression head. This variant applies the same principle of grouped local feature spaces and learned coupling without using the time-dependent cochain update heads.

Algorithm~\ref{alg:generic_csno_training_step} summarizes one supervised training step. The task-specific map \(\mathcal{U}_\theta\) may be an unconstrained residual predictor or a compatibility-preserving update built from coboundary or flux maps.

\begin{algorithm}[t]
\caption{One training step of the Cellular Sheaf Neural Operator}
\label{alg:generic_csno_training_step}
\begin{algorithmic}[1]
\Require Input batch \(x\), target batch \(y\), cell complex \(K\), coboundaries \(\{d_k\}\)
\State Project data to cochains: \(u^n=\Pi_{\mathrm{data}\to K}(x)\)
\State Lift physical cochains to hidden cochains \(h_k^0\), \(k=0,\ldots,d\)
\For{\(\ell=0,\ldots,L-1\)}
    \State Update each \(h_k^\ell\) using within-dimension maps, learned sheaf restrictions, incidence messages, and optional sheaf-Hodge filtering
\EndFor
\State Apply task-specific update heads: \(\widehat{u}^{n+1}=\mathcal{U}_{\theta}(u^n,h^L,K)\)
\State Project back to the dataset representation: \(\widehat{y}=\Pi_{K\to\mathrm{data}}(\widehat{u}^{n+1})\)
\State Compute supervised loss and optional cochain compatibility diagnostics
\State Update parameters by backpropagation
\end{algorithmic}
\end{algorithm}

\section{Experimental design}
\label{sec:experiments}

The experiments evaluate CSNO on time-dependent turbulent MHD and fusion-equilibrium regression. The evaluation includes pointwise prediction error, magnetic-divergence behavior, autoregressive rollout stability, per-field accuracy, spectral behavior, equilibrium-regression accuracy, and computational cost. The numerical study uses two datasets: The Well MHD\_64 and ConStellaration \citep{ohana2024well,cadena2025constellaration}. These datasets serve as complementary MHD and plasma-physics applications of the broader cellular-sheaf surrogate framework.

All time-dependent three-dimensional experiments use the same preprocessed input and target tensors across models. The Well fields are cropped from the native grid before training. The U-Net and FNO baselines operate directly on these channel-first tensors. The Cellular Sheaf Neural Operator first maps the same tensors to cochain-valued fields on a cubical cell complex, applies the cellular-sheaf update internally, and then projects the prediction back to the common tensor representation for supervised training and evaluation.

\subsection{Datasets}
\label{subsec:datasets}

The Well MHD\_64 benchmark provides three-dimensional compressible MHD turbulence derived from the Catalogue for Astrophysical Turbulence Simulations and standardized as part of The Well benchmark collection \citep{burkhart2020catalogue,ohana2024well}. ConStellaration provides quasi-isodynamic stellarator plasma boundaries, optimization benchmarks, and ideal-MHD equilibrium data, and is used here as a fusion-equilibrium surrogate-modeling benchmark \citep{cadena2025constellaration}.

The Well MHD\_64 task is formulated as a supervised time-dependent prediction problem. Given one or more previous field states, the model predicts the next field state. If multiple input frames are used, the temporal history is stacked along the channel dimension before being passed to the model. ConStellaration is treated separately as a supervised equilibrium-regression task. Numeric features parsed from boundary, metric, configuration, and, when joinable, equilibrium records are mapped to selected numeric targets.

For The Well MHD\_64, stored samples have channel-first tensor form
\begin{equation*}
    u^n \in \mathbb{R}^{C \times N_x \times N_y \times N_z},
\end{equation*}
where \((N_x,N_y,N_z)\) denotes the preprocessed spatial grid. The state is constructed from density, magnetic-field, and velocity channels, ordered as density, \(B_x\), \(B_y\), \(B_z\), \(v_x\), \(v_y\), and \(v_z\). Training uses deterministic index-seeded random \(48^3\) crops by default, while validation and testing use deterministic centered crops unless cropping is disabled.

For CSNO, each preprocessed The Well grid is associated with a cubical cell complex. Cell-centered state channels are mapped to volume-cell cochains, magnetic components are mapped to face-based magnetic-flux cochains, and learned electromotive quantities are represented on edges. Predictions are projected back to the common channel-first tensor format so that CSNO, U-Net, and FNO are trained and evaluated against the same targets.

Preprocessed samples are cached locally after field extraction, channel construction, temporal stacking, cropping, normalization metadata construction, and tensor assembly. For CSNO, cubical-complex incidence, coboundary, and Hodge-related data are constructed and reused through an internal complex cache when the grid shape, spacing, periodicity, device, and dtype match. These caches reduce repeated HDF5 parsing, tensor assembly, and cell-complex construction overhead without changing the train, validation, or test splits.

The Well benchmark tests turbulent MHD scalability and diagnostics on a regular cubical complex. ConStellaration tests whether grouped sheaf-style representations and learned restrictions are useful for fusion-relevant equilibrium prediction, although it is not treated as a time-dependent cochain simulation.

\subsection{Baseline methods}
\label{subsec:baselines}

For the time-dependent three-dimensional MHD task, we compare CSNO with a three-dimensional convolutional U-Net \citep{ronneberger2015unet} and a three-dimensional Fourier Neural Operator \citep{li2021fourier,kovachki2023neural}. The U-Net provides a strong local convolutional baseline, while the FNO provides a global spectral neural-operator baseline. Both baselines operate directly on homogeneous channel-first grid tensors.

The Cellular Sheaf Neural Operator differs from these baselines in both representation and update structure. It maps grid tensors to cochain-valued fields on a cell complex, represents magnetic flux on faces, predicts edge-based electromotive quantities, and advances fluid variables through learned face fluxes and cell source terms. Cell-local feature spaces are coupled through incidence-aligned restriction maps and optional sheaf-Hodge-Laplacian message terms. In the current implementation, the magnetic update has the form
\begin{equation*}
    B^{n+1} = B^n - \Delta t\, d_1 E_\theta,
\end{equation*}
and the fluid-state update uses learned face fluxes and cell source terms. This allows the native magnetic update to be checked against the cochain identity \(d_2d_1=0\). These design choices make CSNO a comparison against U-Net and FNO not only at the level of neural backbone, but also at the level of computational representation and update parameterization.

For ConStellaration, the baseline is a fully connected regression model that maps flattened numeric features to equilibrium targets. The sheaf-style equilibrium model instead partitions the flattened feature vector into learned fibers, couples those fibers through restriction-style maps, and then predicts the selected equilibrium quantities.

\subsection{Evaluation metrics}
\label{subsec:metrics}

The primary accuracy metrics are mean-squared error, mean absolute error, and relative \(L^2\) error:
\begin{equation}
    \mathrm{MSE}(\widehat{u},u)
    =
    \frac{1}{|\Omega|C}
    \sum_{c=1}^C \sum_{x \in \Omega}
    \left|\widehat{u}_c(x)-u_c(x)\right|^2,
\end{equation}
\begin{equation}
    \mathrm{RelL2}(\widehat{u},u)
    =
    \frac{\|\widehat{u}-u\|_2}{\|u\|_2}.
\end{equation}
We also report per-channel or per-target relative errors for the variables present in each dataset, including density, magnetic-field components, velocity components, and ConStellaration equilibrium-related targets. For cropped data, all grid-based metrics are computed on the preprocessed grid used for model input and output.

Because magnetic-divergence control is central to the MHD application, we report grid-based magnetic-divergence diagnostics for all time-dependent models when magnetic channels are available. The diagnostic is computed after predictions are represented in channel-first tensor form:
\begin{equation}
    \nabla_h \cdot B
    =
    D_x B_x + D_y B_y + D_z B_z.
\end{equation}
The current implementation uses periodic finite differences for this diagnostic. We report the root-mean-square magnetic-divergence magnitude and a relative magnetic-divergence measure normalized by magnetic-field magnitude. These diagnostics are motivated by the known role of numerical magnetic divergence in producing nonphysical MHD behavior \citep{brackbill1980effect,toth2000divb}.

For CSNO, we additionally verify the native cochain-level magnetic update. Since
\begin{equation}
    d_2 B^{n+1} - d_2 B^n
    =
    -\Delta t\, d_2 d_1 E_\theta,
\end{equation}
the preservation error should be controlled by the discrete identity \(d_2d_1=0\). This verification checks the structure-preserving update in the model's native representation. The main cross-model divergence metrics are still computed on projected grid tensors so that CSNO, U-Net, and FNO are compared in a common representation.

Temporal stability is evaluated through autoregressive rollouts for The Well MHD\_64. Starting from an initial state, predictions are fed back as inputs to generate future states. We report rollout relative \(L^2\) error over time, rollout MSE over time, rollout magnetic divergence over time, mean rollout relative \(L^2\) error, final-step relative \(L^2\) error, and whether the rollout develops non-finite values or exceeds a configured large-value instability threshold. The implementation also reports an energy-like one-step drift based on \(0.5\,\mathrm{mean}(u^2)\) relative to the target value. It does not, by default, compute a physically decomposed kinetic-plus-magnetic rollout drift relative to the initial condition.

For turbulent or multiscale fields, we compute a spectral diagnostic when the prediction is a three-dimensional grid tensor. The implemented metric compares isotropic three-dimensional power spectra computed from FFTs \citep{takamoto2022pdebench,ohana2024well}. For ConStellaration, we report mean-squared error, mean absolute error, relative \(L^2\) error, and per-target relative error. We also report computational metrics available from the runner, including parameter count, inference time per batch, epoch timing, throughput in samples per second, and GPU memory allocation and reservation when CUDA is available.

\subsection{Statistical protocol}
\label{subsec:statistics}

The final reported experiments are repeated over multiple random seeds. Unless otherwise stated, final comparisons use ten seeds; reduced exploratory or fast-development runs may use fewer seeds and are labeled accordingly. Each seed controls Python, NumPy, and PyTorch random number generation, data-loader shuffling, model initialization, and stochastic training operations. For each seed, the runner saves both the best validation checkpoint and the final checkpoint. In the current experiment script, validation and test metrics are computed from the trained model instance after fitting unless an analysis script explicitly reloads the best validation checkpoint. For scalar metrics, we report mean performance with variability across seeds. Paired comparisons are performed across matching seeds between the Cellular Sheaf Neural Operator and each baseline. For each paired comparison, we report the mean paired difference, relative percentage improvement, bootstrap confidence interval, and paired significance test when appropriate. Deterministic bootstrap confidence intervals over seeds are computed using up to \(10{,}000\) resamples. Raw per-seed metrics are saved to support reproducibility and independent post-processing.

\subsection{Hardware and software}
\label{subsec:hardware_software}

All experiments were run on a local Linux workstation with an NVIDIA RTX PRO 6000 Blackwell Workstation Edition GPU, an Intel Core Ultra 9 285K CPU, and 128 GB of DDR5 ECC memory. Experiments were implemented in Python using PyTorch with CUDA acceleration \citep{paszke2019pytorch}. Models were trained in single-GPU mode using bfloat16 automatic mixed precision on CUDA, with parameters and optimizer states kept in full precision. TensorFloat-32 operations were enabled where supported. PyTorch compilation was enabled for selected dense baselines and disabled by default for FNO, CSNO, and the sheaf-equilibrium model. All timing and memory measurements were collected on the same workstation using the same software environment.

The software stack included Python, PyTorch, CUDA, NumPy, SciPy, pandas, h5py, tqdm, and Matplotlib. Sparse cell-complex operations used PyTorch sparse operations, with mixed precision disabled around sparse matrix multiplications when required. The Well MHD\_64 data were loaded from local HDF5 files, and the ConStellaration subset was loaded from local JSONL records. Cached tensor samples, optimized data loaders, and cached cell-complex incidence, coboundary, and geometric data were used when compatible. Code, preprocessing scripts, configurations, and reproducibility scripts will be made publicly available upon publication.

\section{Numerical results}
\label{sec:results}

\subsection{One-step prediction accuracy}
\label{subsec:one_step_results}

Table~\ref{tab:wells_main_metrics} and Fig.~\ref{fig:wells_aggregate_summary} show that U-Net gives the lowest one-step pointwise error on The Well MHD\_64. Its mean MSE, MAE, and relative \(L^2\) error are lower than both neural-operator models. CSNO therefore should not be interpreted as a uniformly better short-horizon field predictor than the convolutional baseline.

The comparison with FNO is more favorable. CSNO reduces one-step MSE from \(1.342\) to \(1.234\), MAE from \(0.430\) to \(0.417\), and relative \(L^2\) error from \(0.405\) to \(0.388\). These correspond to relative reductions of approximately \(8.0\%\), \(2.8\%\), and \(4.4\%\), respectively. Thus, CSNO improves over FNO on ordinary one-step prediction metrics, while U-Net remains substantially stronger on pointwise reconstruction.

\begin{table}[b]
\centering
\caption{Aggregate The Well MHD\_64 metrics over ten random seeds. Values are mean with two-sided \(95\%\) confidence intervals in parentheses. Lower is better for all metrics}
\label{tab:wells_main_metrics}
\resizebox{\linewidth}{!}{
\begin{tabular}{lccc}
\toprule
Metric & U-Net & FNO & CSNO \\
\midrule
MSE
& \(0.329\;(0.325,0.333)\)
& \(1.342\;(1.338,1.346)\)
& \(1.234\;(1.229,1.240)\) \\
MAE
& \(0.250\;(0.248,0.252)\)
& \(0.430\;(0.429,0.430)\)
& \(0.417\;(0.417,0.418)\) \\
Relative \(L^2\)
& \(0.209\;(0.208,0.211)\)
& \(0.405\;(0.405,0.406)\)
& \(0.388\;(0.387,0.388)\) \\
Magnetic divergence \(L^2\)
& \(0.250\;(0.249,0.252)\)
& \(0.242\;(0.242,0.243)\)
& \(0.235\;(0.235,0.236)\) \\
Relative magnetic divergence
& \(0.203\;(0.202,0.203)\)
& \(0.202\;(0.201,0.202)\)
& \(0.195\;(0.195,0.195)\) \\
Spectral error
& \(0.0814\;(0.0761,0.0867)\)
& \(0.0465\;(0.0408,0.0523)\)
& \(0.0233\;(0.0228,0.0237)\) \\
Mean rollout relative \(L^2\)
& \(0.608\;(0.605,0.612)\)
& \(0.510\;(0.508,0.512)\)
& \(0.486\;(0.484,0.487)\) \\
Final-step rollout relative \(L^2\)
& \(0.899\;(0.893,0.905)\)
& \(0.625\;(0.620,0.629)\)
& \(0.591\;(0.587,0.594)\) \\
\bottomrule
\end{tabular}}
\end{table}

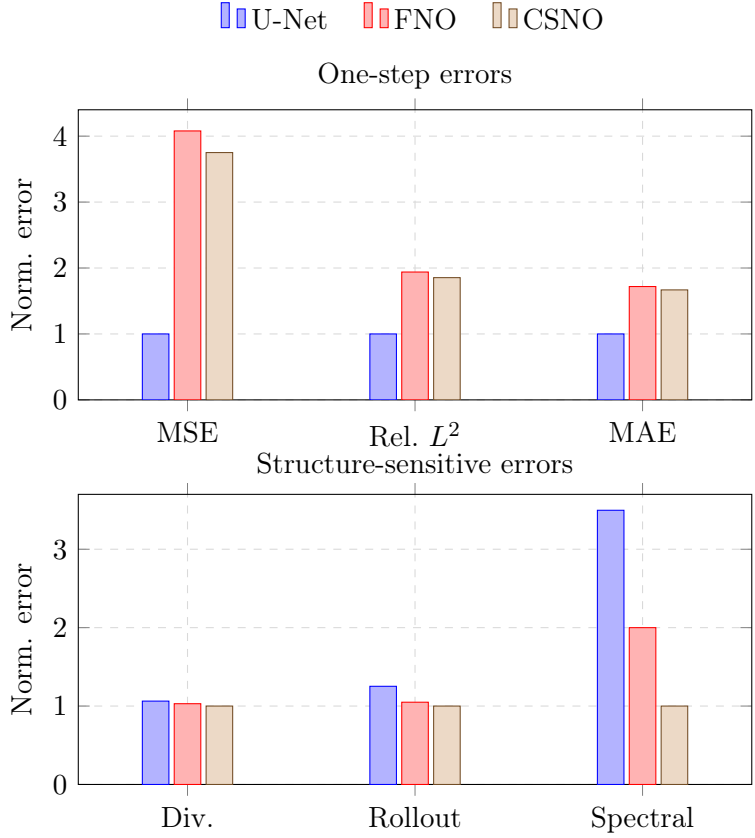
\begin{figure}[t]
\centering
\begin{tikzpicture}
\begin{groupplot}[
    group style={
        group size=1 by 2,
        vertical sep=1.25cm,
        group name=summaryplots
    },
    ybar,
    scale only axis,
    width=0.65\linewidth,
    height=0.28\linewidth,
    ymin=0,
    ylabel={Norm. error},
    enlarge x limits=0.24,
    grid=major,
    grid style={dashed,gray!30},
    tick label style={font=\small},
    label style={font=\small},
    title style={font=\small, yshift=-2pt},
    xticklabel style={font=\small},
]

\nextgroupplot[
    title={One-step errors},
    ymax=4.4,
    symbolic x coords={MSE,Rel.\ $L^2$,MAE},
    xtick=data,
    legend to name=summarylegend,
    legend columns=3,
    legend style={
        draw=none,
        font=\small,
        /tikz/every even column/.append style={column sep=0.4cm}
    }
]
\addplot+[ybar, bar width=10pt] coordinates {
    (MSE,1.000)
    (Rel.\ $L^2$,1.000)
    (MAE,1.000)
};
\addplot+[ybar, bar width=10pt] coordinates {
    (MSE,4.079)
    (Rel.\ $L^2$,1.939)
    (MAE,1.719)
};
\addplot+[ybar, bar width=10pt] coordinates {
    (MSE,3.751)
    (Rel.\ $L^2$,1.854)
    (MAE,1.668)
};
\legend{U-Net,FNO,CSNO}

\nextgroupplot[
    title={Structure-sensitive errors},
    ymax=3.7,
    symbolic x coords={Div.,Rollout,Spectral},
    xtick=data,
]
\addplot+[ybar, bar width=10pt] coordinates {
    (Div.,1.063)
    (Rollout,1.252)
    (Spectral,3.497)
};
\addplot+[ybar, bar width=10pt] coordinates {
    (Div.,1.030)
    (Rollout,1.049)
    (Spectral,2.000)
};
\addplot+[ybar, bar width=10pt] coordinates {
    (Div.,1.000)
    (Rollout,1.000)
    (Spectral,1.000)
};

\end{groupplot}

\node[anchor=south] at ($(summaryplots c1r1.north)+(0,0.72cm)$)
{\pgfplotslegendfromname{summarylegend}};
\end{tikzpicture}
\caption{Aggregate The Well MHD\_64 diagnostics, normalized by the best model for each metric. Lower is better}
\label{fig:wells_aggregate_summary}
\end{figure}

\begin{figure}[t]
    \centering
    \begin{tabular}{c}
        \includegraphics[width=0.65\linewidth]{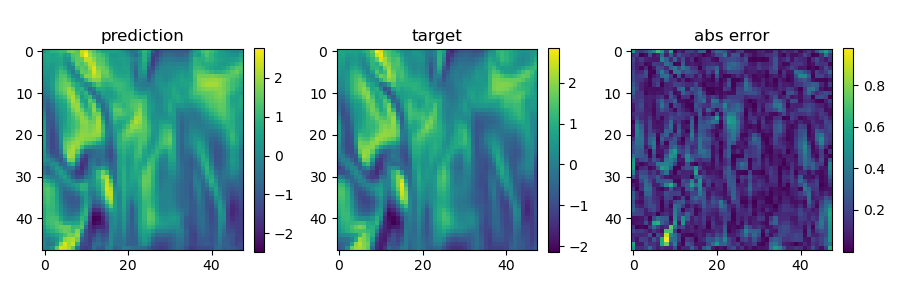}
        \\
        U\text{-}Net
        \\[0.1em]
        \includegraphics[width=0.65\linewidth]{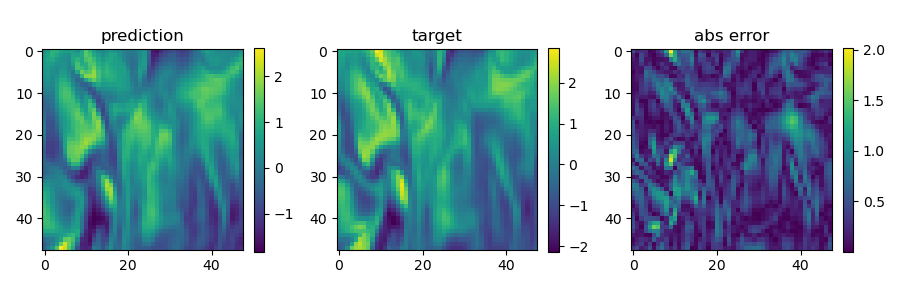}
        \\
        FNO
        \\[0.1em]
        \includegraphics[width=0.65\linewidth]{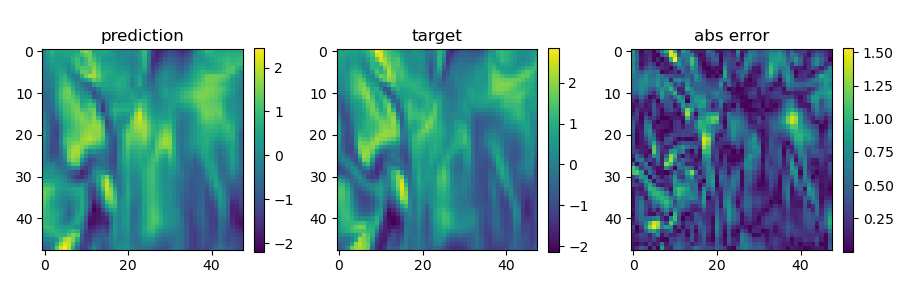}
        \\
        CSNO
    \end{tabular}
    \caption{Representative one-step prediction examples from seed 0 on The Well MHD\_64. All panels use the same visualization format produced by the evaluation pipeline}
    \label{fig:wells_prediction_examples}
\end{figure}

\subsection{Magnetic divergence control}
\label{subsec:divergence_results}

CSNO gives the lowest aggregate magnetic-divergence diagnostics among the three time-dependent models. Its mean grid-space magnetic-divergence \(L^2\) value is \(0.235\), compared with \(0.242\) for FNO and \(0.250\) for U-Net. The relative magnetic-divergence values are \(0.195\) for CSNO, \(0.202\) for FNO, and \(0.203\) for U-Net. This corresponds to a relative-divergence reduction of approximately \(3.3\%\) against FNO and \(3.7\%\) against U-Net.

These improvements are modest but consistent with the intended role of the architecture. The reported grid-space divergence is computed after all predictions are represented in the common tensor format. For CSNO, this is distinct from the native cochain-level update, where the magnetic field is advanced so that \(d_2B^{n+1}=d_2B^n\) up to numerical precision. The grid-space diagnostic is therefore an end-to-end test that includes both the learned update and the projection between cochains and stored tensor fields.

\begin{figure}[t]
\centering
\begin{tikzpicture}
\begin{axis}[
    ybar,
    width=0.65\linewidth,
    height=0.42\linewidth,
    ymin=0.23,
    ymax=0.255,
    ylabel={Magnetic-divergence \(L^2\)},
    symbolic x coords={U-Net,FNO,CSNO},
    xtick=data,
    bar width=40pt,
    enlarge x limits=0.25,
    grid=major,
    grid style={dashed,gray!30},
    tick label style={font=\small},
    label style={font=\small},
    ytick={0.23,0.235,0.24,0.245,0.25,0.255},
    error bars/y dir=both,
    error bars/y explicit,
]
\addplot+[
    error bars/.cd,
    y dir=both,
    y explicit
]
coordinates {
    (U-Net,0.250204) +- (0,0.001518)
    (FNO,0.242386) +- (0,0.000542)
    (CSNO,0.235430) +- (0,0.000552)
};
\end{axis}
\end{tikzpicture}
\caption{Aggregate magnetic-divergence \(L^2\) diagnostic on The Well MHD\_64 over ten random seeds. Error bars show two-sided \(95\%\) confidence intervals. CSNO obtains the lowest mean magnetic-divergence error among the three models}
\label{fig:wells_divergence}
\end{figure}
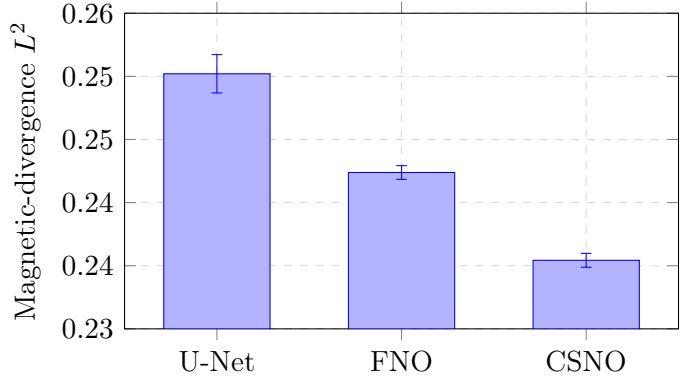

\subsection{Rollout stability}
\label{subsec:rollout_results}

Autoregressive rollout is the clearest empirical advantage for CSNO on The Well MHD\_64. The mean rollout relative \(L^2\) error is \(0.486\) for CSNO, compared with \(0.510\) for FNO and \(0.608\) for U-Net. The final-step rollout relative \(L^2\) error is \(0.591\) for CSNO, \(0.625\) for FNO, and \(0.899\) for U-Net. CSNO therefore reduces mean rollout error by approximately \(4.7\%\) relative to FNO and \(20.1\%\) relative to U-Net. At the final rollout step, the corresponding reductions are \(5.5\%\) relative to FNO and \(34.3\%\) relative to U-Net.

This result separates one-step interpolation quality from dynamical usefulness. Although U-Net gives the lowest one-step error, its autoregressive error grows much more severely. CSNO has higher one-step pointwise error than U-Net, but produces more stable rollout behavior. This supports the main computational claim of the paper: a surrogate organized around cochain placement and compatibility-preserving updates can improve long-horizon behavior even when it does not minimize one-step tensor error.

\begin{figure}[t]
\centering
\begin{tikzpicture}
\begin{axis}[
    ybar,
    width=0.65\linewidth,
    height=0.42\linewidth,
    ymin=0.4,
    ymax=0.65,
    ylabel={Mean rollout relative \(L^2\)},
    symbolic x coords={U-Net,FNO,CSNO},
    xtick=data,
    bar width=40pt,
    enlarge x limits=0.25,
    grid=major,
    grid style={dashed,gray!30},
    tick label style={font=\small},
    label style={font=\small},
    ytick={0.4,0.45,0.5,0.55,0.6,0.65},
    error bars/y dir=both,
    error bars/y explicit,
]
\addplot+[
    error bars/.cd,
    y dir=both,
    y explicit
]
coordinates {
    (U-Net,0.608432) +- (0,0.003112)
    (FNO,0.509792) +- (0,0.001978)
    (CSNO,0.485934) +- (0,0.001552)
};
\end{axis}
\end{tikzpicture}
\caption{Aggregate autoregressive rollout behavior on The Well MHD\_64 over ten random seeds. Error bars show two-sided \(95\%\) confidence intervals. CSNO obtains the lowest mean rollout relative \(L^2\) error despite not having the lowest one-step pointwise error}
\label{fig:wells_rollout}
\end{figure}
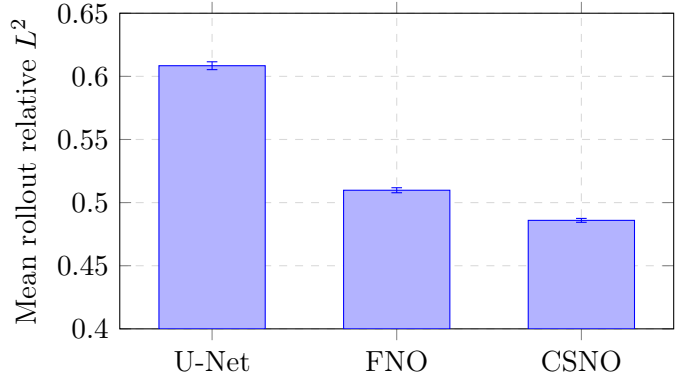

\subsection{Spectral and physical diagnostics}
\label{subsec:spectral_physical}

The spectral diagnostic also favors CSNO. The mean isotropic three-dimensional spectral error is \(0.0233\) for CSNO, compared with \(0.0465\) for FNO and \(0.0814\) for U-Net. CSNO therefore reduces spectral error by approximately \(50.0\%\) relative to FNO and \(71.4\%\) relative to U-Net. This is notable because FNO is explicitly built around Fourier-domain operators and might be expected to have an advantage on spectral quantities.

The energy-like one-step drift gives a more mixed picture. U-Net obtains the lowest value, \(0.106\), while CSNO obtains \(0.155\) and FNO obtains \(0.177\). CSNO improves over FNO by approximately \(12.8\%\), but does not improve over U-Net on this metric. Therefore, the physical diagnostics do not all favor the same model. CSNO is not uniformly dominant across all physical diagnostics, but improves the metrics most closely tied to the proposed structure-preserving design, especially divergence, spectral error, and rollout behavior.

\begin{figure}[ht]
\centering
\begin{tikzpicture}
\begin{axis}[
    ybar,
    width=0.65\linewidth,
    height=0.42\linewidth,
    ymin=0,
    ymax=0.105,
    ylabel={Spectral error},
    symbolic x coords={U-Net,FNO,CSNO},
    xtick=data,
    bar width=40pt,
    enlarge x limits=0.25,
    grid=major,
    grid style={dashed,gray!30},
    tick label style={font=\small},
    label style={font=\small},
    ytick={0,0.02,0.04,0.06,0.08,0.10},
    error bars/y dir=both,
    error bars/y explicit,
]
\addplot+[
    error bars/.cd,
    y dir=both,
    y explicit
]
coordinates {
    (U-Net,0.081396) +- (0,0.005262)
    (FNO,0.046542) +- (0,0.005730)
    (CSNO,0.023274) +- (0,0.000466)
};
\end{axis}
\end{tikzpicture}
\caption{Aggregate isotropic three-dimensional spectral error on The Well MHD\_64 over ten random seeds. Error bars show two-sided \(95\%\) confidence intervals. CSNO obtains the lowest spectral error}
\label{fig:wells_spectra}
\end{figure}
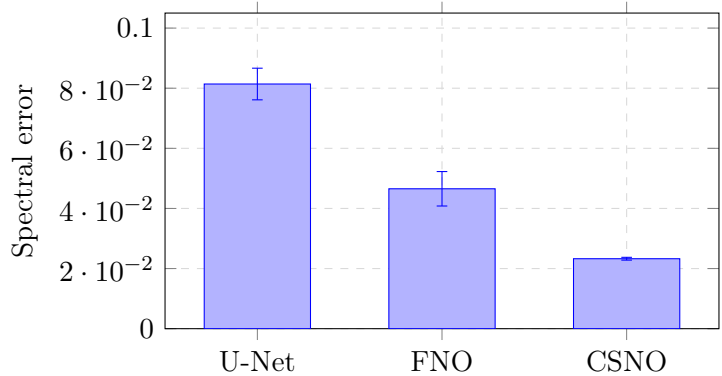

\subsection{Fusion-equilibrium regression}
\label{subsec:constellaration_results}

ConStellaration is used to evaluate whether the sheaf-style grouped representation is useful beyond time-dependent grid prediction. On this problem, the sheaf-equilibrium model substantially outperforms the fully connected MLP baseline. Table~\ref{tab:constellaration_results} summarizes the results. The MSE decreases from \(0.0193\) to \(0.00548\), a reduction of approximately \(71.6\%\). The MAE decreases from \(0.0664\) to \(0.0296\), and relative \(L^2\) error decreases from \(0.129\) to \(0.0643\). The energy-like drift metric also improves from \(0.0256\) to \(0.0174\).

These results suggest that learned grouping and restriction-style coupling are useful for structured equilibrium-regression tasks. The sheaf-equilibrium model uses more parameters, with approximately \(1.42\) million parameters compared with \(231\) thousand for the MLP. However, the measured inference time per batch is essentially unchanged, with both models near \(2.97\) in the recorded units.

\begin{table}[t]
\centering
\caption{ConStellaration equilibrium-regression metrics over ten random seeds. Values are mean with two-sided \(95\%\) confidence intervals in parentheses. Lower is better for all metrics}
\label{tab:constellaration_results}
\resizebox{\linewidth}{!}{
\begin{tabular}{lcc}
\toprule
Metric & MLP & Sheaf-equilibrium \\
\midrule
MSE
& \(0.0193\;(0.0184,0.0202)\)
& \(0.00548\;(0.00526,0.00570)\) \\
MAE
& \(0.0664\;(0.0658,0.0670)\)
& \(0.0296\;(0.0292,0.0299)\) \\
Relative \(L^2\)
& \(0.129\;(0.126,0.131)\)
& \(0.0643\;(0.0632,0.0654)\) \\
Energy-like drift
& \(0.0256\;(0.0221,0.0290)\)
& \(0.0174\;(0.0166,0.0182)\) \\
Inference time per batch
& \(2.975\;(2.963,2.986)\)
& \(2.972\;(2.964,2.981)\) \\
Parameter count
& \(230{,}927\)
& \(1{,}419{,}279\) \\
\bottomrule
\end{tabular}}
\end{table}

\subsection{Computational cost}
\label{subsec:computational_cost}

CSNO is substantially more parameter-efficient than the dense baselines on The Well MHD\_64. It uses \(60{,}393\) parameters, compared with \(2.10\) million for FNO and \(1.29\) million for U-Net. This corresponds to approximately \(97.1\%\) fewer parameters than FNO and \(95.3\%\) fewer parameters than U-Net.

This parameter efficiency does not translate directly into faster inference. The mean inference time per batch is \(0.0212\) for CSNO, compared with \(0.0128\) for FNO and \(0.0109\) for U-Net. Thus, CSNO is approximately \(1.65\times\) slower than FNO and \(1.94\times\) slower than U-Net in the measured inference loop. This reflects the overhead of cochain projection, sparse incidence operations, learned restrictions, and Hodge-weighted message passing. The practical value of CSNO therefore depends on the application. If the goal is only minimum one-step tensor error or fastest inference, U-Net is preferable in this experiment. If the goal is improved rollout behavior, spectral fidelity, divergence diagnostics, and compact parameterization, CSNO provides a favorable tradeoff.

\begin{table}[t]
\centering
\caption{The Well MHD\_64 computational cost. CSNO is substantially smaller than both dense baselines, but has higher measured inference time per batch}
\label{tab:wells_cost}
\begin{tabular}{lccc}
\toprule
Model & Parameters & Inference time per batch & Relative size vs. CSNO \\
\midrule
U-Net & \(1{,}290{,}343\) & \(0.0109\) & \(21.4\times\) \\
FNO & \(2{,}102{,}919\) & \(0.0128\) & \(34.8\times\) \\
CSNO & \(60{,}393\) & \(0.0212\) & \(1.0\times\) \\
\bottomrule
\end{tabular}
\end{table}

Altogether, the results highlight CSNO as a structure-aware surrogate designed for stable, physically organized prediction rather than only for minimizing one-step error. Although U-Net is stronger on one-step pointwise accuracy and the dense baselines are faster at inference, CSNO delivers better rollout behavior, spectral fidelity, magnetic-divergence diagnostics, and parameter efficiency. 

\section{Discussion}
\label{sec:discussion}

The results point to a broader issue in learned PDE surrogate modeling; a model which is accurate as a one-step regressor is not necessarily the best learned time-stepper. For applications such as forecasting, control, design iteration, and uncertainty propagation, the surrogate is often used autoregressively, so accumulated error, compatibility drift, spectral distortion, and constraint violation become more important than pointwise reconstruction. CSNO is designed to access the algebraic organization of the discretization rather than treating the state as a tensor of channels.

The magnetic-divergence results illustrate the value of this representation. The gain in projected grid-space divergence is modest, but the native magnetic update has a stronger interpretation because it factors through the edge-to-face coboundary. This separates two questions that are often conflated in neural surrogate evaluation: whether the architecture preserves a discrete compatibility relation internally, and whether the projected grid diagnostic fully reflects that preservation. For constrained systems, this distinction matters because physically meaningful structure may live in the computational representation rather than in the storage format of the dataset.

The comparison with FNO suggests that spectral parameterization is not the only route to good multiscale behavior. FNO is a natural baseline for grid-based PDE data, but CSNO improves the spectral diagnostic despite not being primarily a Fourier-domain architecture. This indicates that organizing variables by their geometric role can affect how multiscale structure is learned. In MHD, fluxes, circulations, and cell-centered quantities are not interchangeable channels; preserving this organization can improve behavior that is not captured by a generic global convolution alone.

The comparison with U-Net highlights the complementary tradeoff. A convolutional encoder--decoder is highly effective for short-horizon field reconstruction, but its strong one-step accuracy does not automatically translate into better autoregressive dynamics. CSNO gives up pointwise accuracy for stronger rollout, divergence, and spectral behavior, which is more relevant for models used as surrogate time-steppers. The ConStellaration results suggest that the same principle extends beyond MHD rollouts. Grouped feature spaces and learned restriction-style coupling also improve structured equilibrium regression. Taken together, these results argue for evaluation protocols that treat rollout growth, constraint violation, spectral structure, and model size as first-class quantities rather than secondary diagnostics.

\section{Limitations}
\label{sec:limitations}

The empirical results have several limitations. CSNO does not achieve the best one-step pointwise accuracy on The Well MHD\_64, and its projected grid-space magnetic-divergence improvements are modest. The divergence maps are also affected by boundary-sensitive finite-difference artifacts from cropped fields. The native cochain update preserves the corresponding algebraic quantity by construction, but the projected grid diagnostic still depends on interpolation, projection, and boundary handling. CSNO is also slower at inference than both dense baselines despite using far fewer parameters. Finally, the time-dependent validation is centered on The Well MHD\_64; ConStellaration is an equilibrium-regression task rather than a second autoregressive PDE benchmark. These limitations make the current experiments a validation of the cellular-sheaf inductive bias, not a final benchmark of the full generality of the framework. The strongest gains are concentrated in structure-sensitive diagnostics, so the method may be less attractive when ordinary one-step tensor error is the primary objective.

The method also has design limitations. The framework is intended for general oriented cell complexes, but the time-dependent experiments use cubical complexes induced by structured-grid data, so performance on genuinely unstructured production meshes remains untested. Exact preservation claims apply to native cochain updates, not necessarily to every projected grid diagnostic. CSNO also introduces modeling choices beyond U-Net or FNO, including variable placement, learned restriction parameterization, Hodge weights, cochain projection, and task-specific update heads. These choices are meaningful because structure-preserving neural discretizations depend on the chosen complex and function spaces \citep{shaffer2026geonew,kinch2025conditional}, but they require ablation studies. Like other neural operator surrogates, CSNO remains data-dependent: approximation and generalization depend on representation, training distribution, and operator class \citep{lanthaler2022error,adcock2025learning}, and autoregressive use can still accumulate one-step errors \citep{mccabe2023stability}.

Future work should focus on unstructured meshes and non-cubical cell complexes, where the cellular representation is more distinct from ordinary grid-based neural operators. Additional priorities include optimizing sparse and mixed dense--sparse workloads, performing component-wise ablations of restriction maps and update heads, and extending the framework to resistive or Hall MHD, adaptive meshes, boundary-aware projections, and longer rollout horizons.

\section{Conclusions}
\label{sec:conclusion}

This work introduced Cellular Sheaf Neural Operators, a discretization-aware neural operator framework for structure-preserving PDE surrogate modeling. CSNO represents physical states as cochain-valued fields on oriented cell complexes, couples local feature spaces through learned restriction maps, and uses incidence/Hodge-informed message passing to follow the geometry of the computational discretization. Its main architectural contribution is the use of compatibility-preserving update heads, where learned fluxes, circulations, sources, or potentials are routed through coboundary or flux maps so selected constraints are enforced by cell-complex identities.

The MHD instantiation demonstrates this idea in a setting where geometric compatibility is central. Magnetic flux is represented on faces and updated through an edge-based electromotive field, giving a constrained-transport-like magnetic update based on \(d_2d_1=0\). Fluid variables are advanced through learned face fluxes and cell source terms. This does not make CSNO a replacement for a conservative MHD solver, but it gives learned surrogates an update structure closer to compatible discretizations.

The numerical results show that CSNO is strongest on structure-sensitive diagnostics. On The Well MHD\_64, U-Net achieves the lowest one-step pointwise error, but CSNO achieves the best rollout behavior, magnetic-divergence diagnostics, spectral error, and parameter efficiency. Compared with FNO, CSNO improves both one-step prediction metrics and long-horizon diagnostics while using far fewer parameters. On ConStellaration, the sheaf-equilibrium model improves equilibrium-regression accuracy over the MLP baseline, suggesting that learned restriction-style coupling is useful beyond time-dependent MHD prediction.

Overall, cochain-valued representations and compatibility-preserving update heads provide a useful inductive bias for learned PDE surrogates in constrained multiphysics systems. The results also show that evaluation should not rely only on one-step tensor error, since rollout behavior, constraint violation, spectral structure, and parameter efficiency can change model rankings and reveal advantages hidden by short-horizon metrics.

\newpage

\section*{Data availability}

The Well MHD\_64 dataset used in this study is publicly available through The Well benchmark collection. The ConStellaration dataset used for the fusion-equilibrium regression experiments is publicly available from the dataset release described in the cited ConStellaration reference. The processed train, validation, and test splits used in this work are generated from these public datasets by the preprocessing scripts provided with the accompanying code repository.

Code for reproducing the experiments, generating the figures, and computing the reported metrics is available at https://github.com/lennonshikhman/sheaf-neural-operator.
The repository includes the model implementations, preprocessing scripts, experiment configurations, evaluation scripts, and plotting utilities. Intermediate cached tensors, trained model checkpoints, and per-seed metric files can be regenerated from the public datasets using the provided scripts.

\section*{Declaration of competing interest}

The authors declare the following financial interests/personal relationships which may be considered as potential competing interests: computational resources used in this work were provided by Dell Technologies. The authors declare no other known competing financial interests or personal relationships that could have appeared to influence the work reported in this paper.

\section*{CRediT authorship contribution statement}

\textbf{Lennon J. Shikhman:} Conceptualization, Data curation, Formal analysis, Funding acquisition, Investigation, Methodology, Project administration, Resources, Software, Supervision, Validation, Visualization, Writing -- original draft, Writing -- review and editing.

\textbf{Shane Gilbertie:} Investigation, Writing -- review and editing.

\section*{Acknowledgements}


\paragraph{Computational Resources} The authors gratefully acknowledge Dell Technologies, and in particular the Dell Pro Precision division, for providing computational resources that supported the experiments in this work. All experiments were conducted on a Dell Pro Max T2 workstation equipped with an Intel Core Ultra 9 285K processor, 128 GB of DDR5 ECC memory, and an NVIDIA RTX PRO 6000 Blackwell GPU.

\bibliographystyle{elsarticle-num-names}
\bibliography{references}

@article{
    shikhman2026diagnosing,
    title={Diagnosing Failure Modes of Neural Operators Across Diverse {PDE} Families},
    author={Lennon Shikhman},
    journal={Transactions on Machine Learning Research},
    issn={2835-8856},
    year={2026},
    url={https://openreview.net/forum?id=0S1LWZHQYn},
    note={}
}

@inproceedings{ronneberger2015unet,
    author="Ronneberger, Olaf
    and Fischer, Philipp
    and Brox, Thomas",
    editor="Navab, Nassir
    and Hornegger, Joachim
    and Wells, William M.
    and Frangi, Alejandro F.",
    title="U-Net: Convolutional Networks for Biomedical Image Segmentation",
    booktitle="Medical Image Computing and Computer-Assisted Intervention -- MICCAI 2015",
    year="2015",
    publisher="Springer International Publishing",
    address="Cham",
    pages="234--241",
    isbn="978-3-319-24574-4"    
}

@inbook{paszke2019pytorch,
    author = {Paszke, Adam and Gross, Sam and Massa, Francisco and Lerer, Adam and Bradbury, James and Chanan, Gregory and Killeen, Trevor and Lin, Zeming and Gimelshein, Natalia and Antiga, Luca and Desmaison, Alban and K\"{o}pf, Andreas and Yang, Edward and DeVito, Zach and Raison, Martin and Tejani, Alykhan and Chilamkurthy, Sasank and Steiner, Benoit and Fang, Lu and Bai, Junjie and Chintala, Soumith},
    title = {PyTorch: an imperative style, high-performance deep learning library},
    year = {2019},
    publisher = {Curran Associates Inc.},
    address = {Red Hook, NY, USA},
    booktitle = {Proceedings of the 33rd International Conference on Neural Information Processing Systems},
    articleno = {721},
    numpages = {12}
}

@article{burkhart2020catalogue,
  title={The catalogue for astrophysical turbulence simulations (cats)},
  author={Burkhart, B and Appel, SM and Bialy, S and Cho, J and Christensen, AJ and Collins, D and Federrath, Christoph and Fielding, DB and Finkbeiner, D and Hill, AS and others},
  journal={The Astrophysical Journal},
  volume={905},
  number={1},
  pages={14},
  year={2020},
  publisher={IOP Publishing}
}

@misc{cadena2025constellaration,
      title={ConStellaration: A dataset of QI-like stellarator plasma boundaries and optimization benchmarks}, 
      author={Santiago A. Cadena and Andrea Merlo and Emanuel Laude and Alexander Bauer and Atul Agrawal and Maria Pascu and Marija Savtchouk and Enrico Guiraud and Lukas Bonauer and Stuart Hudson and Markus Kaiser},
      year={2025},
      eprint={2506.19583},
      archivePrefix={arXiv},
      primaryClass={cs.LG},
      url={https://arxiv.org/abs/2506.19583}, 
}

@book{leveque2002finite,
  title     = {Finite Volume Methods for Hyperbolic Problems},
  author    = {LeVeque, Randall J.},
  year      = {2002},
  publisher = {Cambridge University Press},
  address   = {Cambridge}
}

@article{cockburn2001runge,
  title   = {Runge--Kutta discontinuous Galerkin methods for convection-dominated problems},
  author  = {Cockburn, Bernardo and Shu, Chi-Wang},
  journal = {Journal of Scientific Computing},
  volume  = {16},
  number  = {3},
  pages   = {173--261},
  year    = {2001}
}

@book{hesthaven2008nodal,
  title     = {Nodal Discontinuous Galerkin Methods: Algorithms, Analysis, and Applications},
  author    = {Hesthaven, Jan S. and Warburton, Tim},
  year      = {2008},
  publisher = {Springer},
  series    = {Texts in Applied Mathematics},
  volume    = {54}
}

@article{hyman1997adjoint,
  title   = {The adjoint operators for the natural discretizations of the divergence, gradient and curl on logically rectangular grids},
  author  = {Hyman, James M. and Shashkov, Mikhail},
  journal = {Applied Numerical Mathematics},
  volume  = {25},
  number  = {4},
  pages   = {413--442},
  year    = {1997}
}

@article{hyman1999orthogonal,
  title   = {The orthogonal decomposition theorems for mimetic finite difference methods},
  author  = {Hyman, James M. and Shashkov, Mikhail},
  journal = {SIAM Journal on Numerical Analysis},
  volume  = {36},
  number  = {3},
  pages   = {788--818},
  year    = {1999}
}

@article{hyman2000james,
    author = {Hyman, James},
    year = {2000},
    month = {08},
    pages = {},
    title = {Mimetic Finite Difference Methods for Maxwell's Equations and the Equations of Magnetic Diffusion},
    volume = {32},
    journal = {Progress In Electromagnetics Research},
    doi = {10.2528/PIER00080104}
}

@article{brezzi2005convergence,
  title   = {Convergence of the mimetic finite difference method for diffusion problems on polyhedral meshes},
  author  = {Brezzi, Franco and Lipnikov, Konstantin and Shashkov, Mikhail},
  journal = {SIAM Journal on Numerical Analysis},
  volume  = {43},
  number  = {5},
  pages   = {1872--1896},
  year    = {2005}
}

@article{arnold2006feec,
  title   = {Finite element exterior calculus, homological techniques, and applications},
  author  = {Arnold, Douglas N. and Falk, Richard S. and Winther, Ragnar},
  journal = {Acta Numerica},
  volume  = {15},
  pages   = {1--155},
  year    = {2006},
  doi     = {10.1017/S0962492906210018}
}

@article{arnold2010feec,
  title   = {Finite element exterior calculus: From Hodge theory to numerical stability},
  author  = {Arnold, Douglas N. and Falk, Richard S. and Winther, Ragnar},
  journal = {Bulletin of the American Mathematical Society},
  volume  = {47},
  number  = {2},
  pages   = {281--354},
  year    = {2010}
}

@misc{desbrun2005dec,
      title={Discrete Exterior Calculus}, 
      author={Mathieu Desbrun and Anil N. Hirani and Melvin Leok and Jerrold E. Marsden},
      year={2005},
      eprint={math/0508341},
      archivePrefix={arXiv},
      primaryClass={math.DG},
      url={https://arxiv.org/abs/math/0508341}, 
}

@phdthesis{hirani2003dec,
  title  = {Discrete Exterior Calculus},
  author = {Hirani, Anil Nirmal},
  school = {California Institute of Technology},
  year   = {2003}
}

@book{bossavit1998computational,
  title     = {Computational Electromagnetism: Variational Formulations, Complementarity, Edge Elements},
  author    = {Bossavit, Alain},
  year      = {1998},
  publisher = {Academic Press},
  address   = {San Diego}
}

@article{evans1988constrained,
    title   = {Simulation of magnetohydrodynamic flows: A constrained transport method},
    author  = {Evans, Charles R. and Hawley, John F.},
    journal = {The Astrophysical Journal},
    volume  = {332},
    pages   = {659--677},
    year    = {1988}
}

@article{toth2000divb,
  title   = {The {$\nabla \cdot B = 0$} constraint in shock-capturing magnetohydrodynamics codes},
  author  = {T{\'o}th, G{\'a}bor},
  journal = {Journal of Computational Physics},
  volume  = {161},
  number  = {2},
  pages   = {605--652},
  year    = {2000},
  doi     = {10.1006/jcph.2000.6519}
}

@article{dedner2002hyperbolic,
  title   = {Hyperbolic divergence cleaning for the MHD equations},
  author  = {Dedner, Andreas and Kemm, Friedemann and Kr{\"o}ner, Dietmar and Munz, Claus-Dieter and Schnitzer, Thomas and Wesenberg, Matthias},
  journal = {Journal of Computational Physics},
  volume  = {175},
  number  = {2},
  pages   = {645--673},
  year    = {2002}
}

@article{balsara2001amr,
   title={Divergence-Free Adaptive Mesh Refinement for Magnetohydrodynamics},
   volume={174},
   ISSN={0021-9991},
   url={http://dx.doi.org/10.1006/jcph.2001.6917},
   DOI={10.1006/jcph.2001.6917},
   number={2},
   journal={Journal of Computational Physics},
   publisher={Elsevier BV},
   author={Balsara, Dinshaw S.},
   year={2001},
   month=Dec, pages={614–648}
}

@article{kovachki2023neural,
  title   = {Neural Operator: Learning Maps Between Function Spaces with Applications to PDEs},
  author  = {Kovachki, Nikola and Li, Zongyi and Liu, Burigede and Azizzadenesheli, Kamyar and Bhattacharya, Kaushik and Stuart, Andrew and Anandkumar, Anima},
  journal = {Journal of Machine Learning Research},
  volume  = {24},
  number  = {89},
  pages   = {1--97},
  year    = {2023}
}

@article{lu2021learning,
   title={Learning nonlinear operators via DeepONet based on the universal approximation theorem of operators},
   volume={3},
   ISSN={2522-5839},
   url={http://dx.doi.org/10.1038/s42256-021-00302-5},
   DOI={10.1038/s42256-021-00302-5},
   number={3},
   journal={Nature Machine Intelligence},
   publisher={Springer Science and Business Media LLC},
   author={Lu, Lu and Jin, Pengzhan and Pang, Guofei and Zhang, Zhongqiang and Karniadakis, George Em},
   year={2021},
   month=Mar,
   pages={218–229}
}

@inproceedings{li2021fourier,
    title={Fourier Neural Operator for Parametric Partial Differential Equations},
    author={Zongyi Li and Nikola Borislavov Kovachki and Kamyar Azizzadenesheli and Burigede liu and Kaushik Bhattacharya and Andrew Stuart and Anima Anandkumar},
    booktitle={International Conference on Learning Representations},
    year={2021},
    url={https://openreview.net/forum?id=c8P9NQVtmnO}
}

@misc{li2020neural,
      title={Neural Operator: Graph Kernel Network for Partial Differential Equations}, 
      author={Zongyi Li and Nikola Kovachki and Kamyar Azizzadenesheli and Burigede Liu and Kaushik Bhattacharya and Andrew Stuart and Anima Anandkumar},
      year={2020},
      eprint={2003.03485},
      archivePrefix={arXiv},
      primaryClass={cs.LG},
      url={https://arxiv.org/abs/2003.03485}, 
}

@inproceedings{pfaff2021learning,
    title={Learning Mesh-Based Simulation with Graph Networks},
    author={Tobias Pfaff and Meire Fortunato and Alvaro Sanchez-Gonzalez and Peter Battaglia},
    booktitle={International Conference on Learning Representations},
    year={2021},
    url={https://openreview.net/forum?id=roNqYL0_XP}
}

@article{li2023fourier,
  title   = {Fourier Neural Operator with Learned Deformations for PDEs on General Geometries},
  author  = {Li, Zongyi and Huang, Daniel Zhengyu and Liu, Burigede and Anandkumar, Anima},
  journal = {Journal of Machine Learning Research},
  volume  = {24},
  number  = {388},
  pages   = {1--26},
  year    = {2023}
}

@article{li2024physics,
    author = {Li, Zongyi and Zheng, Hongkai and Kovachki, Nikola and Jin, David and Chen, Haoxuan and Liu, Burigede and Azizzadenesheli, Kamyar and Anandkumar, Anima},
    title = {Physics-Informed Neural Operator for Learning Partial Differential Equations},
    year = {2024},
    issue_date = {September 2024},
    publisher = {Association for Computing Machinery},
    address = {New York, NY, USA},
    volume = {1},
    number = {3},
    url = {https://doi.org/10.1145/3648506},
    doi = {10.1145/3648506},
    journal = {ACM / IMS J. Data Sci.},
    month = may,
    articleno = {9},
    numpages = {27}
}

@inproceedings{takamoto2022pdebench,
    author = {Takamoto, Makoto and Praditia, Timothy and Leiteritz, Raphael and MacKinlay, Dan and Alesiani, Francesco and Pfl\"{u}ger, Dirk and Niepert, Mathias},
    title = {PDEBENCH: an extensive benchmark for scientific machine learning},
    year = {2022},
    isbn = {9781713871088},
    publisher = {Curran Associates Inc.},
    address = {Red Hook, NY, USA},
    booktitle = {Proceedings of the 36th International Conference on Neural Information Processing Systems},
    articleno = {117},
    numpages = {16},
    location = {New Orleans, LA, USA},
    series = {NIPS '22}
}

@inproceedings{ohana2024well,
    author = {Ohana, Ruben and McCabe, Michael and Meyer, Lucas and Morel, Rudy and Agocs, Fruzsina J. and Beneitez, Miguel and Berger, Marsha and Burkhart, Blakesley and Dalziel, Stuart B. and Fielding, Drummond B. and Fortunato, Daniel and Goldberg, Jared A. and Hirashima, Keiya and Jiang, Yan-Fei and Kerswell, Rich R. and Maddu, Suryanarayana and Miller, Jonah and Mukhopadhyay, Payel and Nixon, Stefan S. and Shen, Jeff and Watteaux, Romain and Blancard, Bruno R\'{e}galdo-Saint and Rozet, Fran\c{c}ois and Parker, Liam and Cranmer, Miles and Ho, Shirley},
    booktitle = {Advances in Neural Information Processing Systems},
    doi = {10.52202/079017-1430},
    editor = {A. Globerson and L. Mackey and D. Belgrave and A. Fan and U. Paquet and J. Tomczak and C. Zhang},
    pages = {44989--45037},
    publisher = {Curran Associates, Inc.},
    title = {The Well: a Large-Scale Collection of Diverse Physics Simulations for Machine Learning},
    url = {https://proceedings.neurips.cc/paper_files/paper/2024/file/4f9a5acd91ac76569f2fe291b1f4772b-Paper-Datasets_and_Benchmarks_Track.pdf},
    volume = {37},
    year = {2024}
}

@inproceedings{kipf2017semi,
    title={Semi-Supervised Classification with Graph Convolutional Networks},
    author={Thomas N. Kipf and Max Welling},
    booktitle={International Conference on Learning Representations},
    year={2017},
    url={https://openreview.net/forum?id=SJU4ayYgl}
}

@inproceedings{gilmer2017neural,
    title     = {Neural Message Passing for Quantum Chemistry},
    author    = {Gilmer, Justin and Schoenholz, Samuel S. and Riley, Patrick F. and Vinyals, Oriol and Dahl, George E.},
    booktitle = {International Conference on Machine Learning},
    pages     = {1263--1272},
    year      = {2017},
    publisher = {PMLR}
}

@InProceedings{sanchez2018graph,
    title = 	 {Graph Networks as Learnable Physics Engines for Inference and Control},
    author =       {Sanchez-Gonzalez, Alvaro and Heess, Nicolas and Springenberg, Jost Tobias and Merel, Josh and Riedmiller, Martin and Hadsell, Raia and Battaglia, Peter},
    booktitle = 	 {Proceedings of the 35th International Conference on Machine Learning},
    pages = 	 {4470--4479},
    year = 	 {2018},
    editor = 	 {Dy, Jennifer and Krause, Andreas},
    volume = 	 {80},
    series = 	 {Proceedings of Machine Learning Research},
    month = 	 {10--15 Jul},
    publisher =    {PMLR},
    url = 	 {https://proceedings.mlr.press/v80/sanchez-gonzalez18a.html}
}

@misc{bronstein2021geometric,
      title={Geometric Deep Learning: Grids, Groups, Graphs, Geodesics, and Gauges}, 
      author={Michael M. Bronstein and Joan Bruna and Taco Cohen and Petar Veličković},
      year={2021},
      eprint={2104.13478},
      archivePrefix={arXiv},
      primaryClass={cs.LG},
      url={https://arxiv.org/abs/2104.13478}, 
}

@misc{hajij2023topological,
      title={Topological Deep Learning: Going Beyond Graph Data}, 
      author={Mustafa Hajij and Ghada Zamzmi and Theodore Papamarkou and Nina Miolane and Aldo Guzmán-Sáenz and Karthikeyan Natesan Ramamurthy and Tolga Birdal and Tamal K. Dey and Soham Mukherjee and Shreyas N. Samaga and Neal Livesay and Robin Walters and Paul Rosen and Michael T. Schaub},
      year={2023},
      eprint={2206.00606},
      archivePrefix={arXiv},
      primaryClass={cs.LG},
      url={https://arxiv.org/abs/2206.00606}, 
}

@misc{papillon2024architectures,
      title={Architectures of Topological Deep Learning: A Survey of Message-Passing Topological Neural Networks}, 
      author={Mathilde Papillon and Sophia Sanborn and Mustafa Hajij and Nina Miolane},
      year={2024},
      eprint={2304.10031},
      archivePrefix={arXiv},
      primaryClass={cs.LG},
      url={https://arxiv.org/abs/2304.10031}, 
}

@InProceedings{papamarkou2024topological,
  title = 	 {Position: Topological Deep Learning is the New Frontier for Relational Learning},
  author =       {Papamarkou, Theodore and Birdal, Tolga and Bronstein, Michael M. and Carlsson, Gunnar E. and Curry, Justin and Gao, Yue and Hajij, Mustafa and Kwitt, Roland and Lio, Pietro and Di Lorenzo, Paolo and Maroulas, Vasileios and Miolane, Nina and Nasrin, Farzana and Natesan Ramamurthy, Karthikeyan and Rieck, Bastian and Scardapane, Simone and Schaub, Michael T and Veli\v{c}kovi\'{c}, Petar and Wang, Bei and Wang, Yusu and Wei, Guowei and Zamzmi, Ghada},
  booktitle = 	 {Proceedings of the 41st International Conference on Machine Learning},
  pages = 	 {39529--39555},
  year = 	 {2024},
  editor = 	 {Salakhutdinov, Ruslan and Kolter, Zico and Heller, Katherine and Weller, Adrian and Oliver, Nuria and Scarlett, Jonathan and Berkenkamp, Felix},
  volume = 	 {235},
  series = 	 {Proceedings of Machine Learning Research},
  month = 	 {21--27 Jul},
  publisher =    {PMLR},
  url = 	 {https://proceedings.mlr.press/v235/papamarkou24a.html}
}

@article{ebli2020simplicial,
  title   = {Simplicial Neural Networks},
  author  = {Ebli, Stefania and Defferrard, Micha{\"e}l and Spreemann, Gard},
  journal = {arXiv preprint arXiv:2010.03633},
  year    = {2020}
}

@InProceedings{bodnar2021weisfeiler,
    title = 	 {Weisfeiler and Lehman Go Topological: Message Passing Simplicial Networks},
    author =       {Bodnar, Cristian and Frasca, Fabrizio and Wang, Yuguang and Otter, Nina and Montufar, Guido F and Li{\'o}, Pietro and Bronstein, Michael},
    booktitle = 	 {Proceedings of the 38th International Conference on Machine Learning},
    pages = 	 {1026--1037},
    year = 	 {2021},
    editor = 	 {Meila, Marina and Zhang, Tong},
    volume = 	 {139},
    series = 	 {Proceedings of Machine Learning Research},
    month = 	 {18--24 Jul},
    publisher =    {PMLR},
    url = 	 {https://proceedings.mlr.press/v139/bodnar21a.html}
}

@article{hansen2019toward,
  title   = {Toward a Spectral Theory of Cellular Sheaves},
  author  = {Hansen, Jakob and Ghrist, Robert},
  journal = {Journal of Applied and Computational Topology},
  volume  = {3},
  pages   = {315--358},
  year    = {2019},
  doi     = {10.1007/s41468-019-00038-7}
}

@inproceedings{hansen2020sheaf,
    title={Sheaf Neural Networks},
    author={Jakob Hansen and Thomas Gebhart},
    booktitle={TDA {\&} Beyond},
    year={2020},
    url={https://openreview.net/forum?id=GgcgIJsT8HD}
}

@inproceedings{bodnar2022neural_sheaf_diffusion,
    author = {Bodnar, Cristian and Di Giovanni, Francesco and Chamberlain, Benjamin and Li\'{o}, Pietro and Bronstein, Michael},
    booktitle = {Advances in Neural Information Processing Systems},
    editor = {S. Koyejo and S. Mohamed and A. Agarwal and D. Belgrave and K. Cho and A. Oh},
    pages = {18527--18541},
    publisher = {Curran Associates, Inc.},
    title = {Neural Sheaf Diffusion: A Topological Perspective on Heterophily and Oversmoothing in GNNs},
    url = {https://proceedings.neurips.cc/paper_files/paper/2022/file/75c45fca2aa416ada062b26cc4fb7641-Paper-Conference.pdf},
    volume = {35},
    year = {2022}
}

@InProceedings{barbero2022sheaf_connection,
    title = 	 {Sheaf Neural Networks with Connection Laplacians},
    author =       {Barbero, Federico and Bodnar, Cristian and S\'aez de Oc\'ariz Borde, Haitz and Bronstein, Michael and Veli\v{c}kovi\'c, Petar and Li\`o, Pietro},
    booktitle = 	 {Proceedings of Topological, Algebraic, and Geometric Learning Workshops 2022},
    pages = 	 {28--36},
    year = 	 {2022},
    editor = 	 {Cloninger, Alexander and Doster, Timothy and Emerson, Tegan and Kaul, Manohar and Ktena, Ira and Kvinge, Henry and Miolane, Nina and Rieck, Bastian and Tymochko, Sarah and Wolf, Guy},
    volume = 	 {196},
    series = 	 {Proceedings of Machine Learning Research},
    month = 	 {25 Feb--22 Jul},
    publisher =    {PMLR},
    url = 	 {https://proceedings.mlr.press/v196/barbero22a.html}
}

@inproceedings{bodnar2021cellular,
    title={Weisfeiler and Lehman Go Cellular: {CW} Networks},
    author={Cristian Bodnar and Fabrizio Frasca and Nina Otter and Yu Guang Wang and Pietro Li{\`o} and Guido Montufar and Michael M. Bronstein},
    booktitle={Advances in Neural Information Processing Systems},
    editor={A. Beygelzimer and Y. Dauphin and P. Liang and J. Wortman Vaughan},
    year={2021},
    url={https://openreview.net/forum?id=uVPZCMVtsSG}
}

@inproceedings{hajij2020cell_complex,
    title={Cell Complex Neural Networks},
    author={Mustafa Hajij and Kyle Istvan and Ghada Zamzmi},
    booktitle={TDA {\&} Beyond},
    year={2020},
    url={https://openreview.net/forum?id=6Tq18ySFpGU}
}

@misc{roddenberry2019hodgenet,
    author={Roddenberry, T. Mitchell and Segarra, Santiago},
    booktitle={2019 53rd Asilomar Conference on Signals, Systems, and Computers}, 
    title={HodgeNet: Graph Neural Networks for Edge Data}, 
    year={2019},
    volume={},
    number={},
    pages={220-224},
    doi={10.1109/IEEECONF44664.2019.9049000}
}

@InProceedings{roddenberry2021principled,
  title = 	 {Principled Simplicial Neural Networks for Trajectory Prediction},
  author =       {Roddenberry, T. Mitchell and Glaze, Nicholas and Segarra, Santiago},
  booktitle = 	 {Proceedings of the 38th International Conference on Machine Learning},
  pages = 	 {9020--9029},
  year = 	 {2021},
  editor = 	 {Meila, Marina and Zhang, Tong},
  volume = 	 {139},
  series = 	 {Proceedings of Machine Learning Research},
  month = 	 {18--24 Jul},
  publisher =    {PMLR},
  url = 	 {https://proceedings.mlr.press/v139/roddenberry21a.html}
}

@article{schaub2021higher_order_signal_processing,
    title = {Signal processing on higher-order networks: Livin’ on the edge... and beyond},
    journal = {Signal Processing},
    volume = {187},
    pages = {108149},
    year = {2021},
    issn = {0165-1684},
    doi = {https://doi.org/10.1016/j.sigpro.2021.108149},
    url = {https://www.sciencedirect.com/science/article/pii/S0165168421001870},
    author = {Michael T. Schaub and Yu Zhu and Jean-Baptiste Seby and T. Mitchell Roddenberry and Santiago Segarra}
}

@book{goedbloed2004principles,
  title     = {Principles of Magnetohydrodynamics: With Applications to Laboratory and Astrophysical Plasmas},
  author    = {Goedbloed, J. P. and Poedts, Stefaan},
  year      = {2004},
  publisher = {Cambridge University Press},
  address   = {Cambridge}
}

@book{freidberg2014ideal,
  title     = {Ideal MHD},
  author    = {Freidberg, Jeffrey P.},
  year      = {2014},
  publisher = {Cambridge University Press},
  address   = {Cambridge}
}

@article{brackbill1980effect,
    title = {The Effect of Nonzero {$\nabla \cdot B$} on the Numerical Solution of the Magnetohydrodynamic Equations},
    journal = {Journal of Computational Physics},
    volume = {35},
    number = {3},
    pages = {426-430},
    year = {1980},
    issn = {0021-9991},
    doi = {https://doi.org/10.1016/0021-9991(80)90079-0},
    url = {https://www.sciencedirect.com/science/article/pii/0021999180900790},
    author = {J.U Brackbill and D.C Barnes}
}

@article{powell1999solution,
  title   = {A Solution-Adaptive Upwind Scheme for Ideal Magnetohydrodynamics},
  author  = {Powell, Kenneth G. and Roe, Philip L. and Linde, Timur J. and Gombosi, Tamas I. and De Zeeuw, Darren L.},
  journal = {Journal of Computational Physics},
  volume  = {154},
  number  = {2},
  pages   = {284--309},
  year    = {1999},
  doi     = {10.1006/jcph.1999.6299}
}

@article{miyoshi2005multi,
  title   = {A Multi-State {HLL} Approximate Riemann Solver for Ideal Magnetohydrodynamics},
  author  = {Miyoshi, Takahiro and Kusano, Kanya},
  journal = {Journal of Computational Physics},
  volume  = {208},
  number  = {1},
  pages   = {315--344},
  year    = {2005},
  doi     = {10.1016/j.jcp.2005.02.017}
}

@article{gardiner2008unsplit,
  title   = {An Unsplit {Godunov} Method for Ideal {MHD} via Constrained Transport in Three Dimensions},
  author  = {Gardiner, Thomas A. and Stone, James M.},
  journal = {Journal of Computational Physics},
  volume  = {227},
  number  = {8},
  pages   = {4123--4141},
  year    = {2008},
  doi     = {10.1016/j.jcp.2007.12.017}
}

@article{stone2008athena,
  title   = {Athena: A New Code for Astrophysical {MHD}},
  author  = {Stone, James M. and Gardiner, Thomas A. and Teuben, Peter and Hawley, John F. and Simon, Jacob B.},
  journal = {The Astrophysical Journal Supplement Series},
  volume  = {178},
  number  = {1},
  pages   = {137--177},
  year    = {2008},
  doi     = {10.1086/588755}
}

@inproceedings{maggs2024simplicial,
    title={Simplicial Representation Learning with Neural \$k\$-Forms},
    author={Kelly Maggs and Celia Hacker and Bastian Rieck},
    booktitle={The Twelfth International Conference on Learning Representations},
    year={2024},
    url={https://openreview.net/forum?id=Djw0XhjHZb}
}

@InProceedings{kovac2024empcn,
  title = 	 {E(n) Equivariant Message Passing Cellular Networks},
  author =       {Kova\u{c}, Veljko and Bekkers, Erik and Li\'{o}, Pietro and Eijkelboom, Floor},
  booktitle = 	 {Proceedings of the Geometry-grounded Representation Learning and Generative Modeling Workshop (GRaM)},
  pages = 	 {173--186},
  year = 	 {2024},
  editor = 	 {Vadgama, Sharvaree and Bekkers, Erik and Pouplin, Alison and Kaba, Sekou-Oumar and Walters, Robin and Lawrence, Hannah and Emerson, Tegan and Kvinge, Henry and Tomczak, Jakub and Jegelka, Stephanie},
  volume = 	 {251},
  series = 	 {Proceedings of Machine Learning Research},
  month = 	 {29 Jul},
  publisher =    {PMLR},
  url = 	 {https://proceedings.mlr.press/v251/kovac-24a.html}
}

@misc{shaffer2026geonew,
      title={Structure-Preserving Learning Improves Geometry Generalization in Neural PDEs}, 
      author={Benjamin D. Shaffer and Shawn Koohy and Brooks Kinch and M. Ani Hsieh and Nathaniel Trask},
      year={2026},
      eprint={2602.02788},
      archivePrefix={arXiv},
      primaryClass={cs.LG},
      url={https://arxiv.org/abs/2602.02788}, 
}

@misc{kinch2025conditional,
      title={Structure-Preserving Digital Twins via Conditional Neural Whitney Forms}, 
      author={Brooks Kinch and Benjamin Shaffer and Elizabeth Armstrong and Michael Meehan and John Hewson and Nathaniel Trask},
      year={2025},
      eprint={2508.06981},
      archivePrefix={arXiv},
      primaryClass={cs.LG},
      url={https://arxiv.org/abs/2508.06981}, 
}

@article{lanthaler2022error,
    author = {Lanthaler, Samuel and Mishra, Siddhartha and Karniadakis, George E},
    title = {Error estimates for DeepONets: a deep learning framework in infinite dimensions},
    journal = {Transactions of Mathematics and Its Applications},
    volume = {6},
    number = {1},
    pages = {tnac001},
    year = {2022},
    month = {01},
    issn = {2398-4945},
    doi = {10.1093/imatrm/tnac001},
    url = {https://doi.org/10.1093/imatrm/tnac001},
    eprint = {https://academic.oup.com/imatrm/article-pdf/6/1/tnac001/42785544/tnac001.pdf},
}

@misc{adcock2025learning,
      title={The Sample Complexity of Learning Lipschitz Operators with respect to Gaussian Measures}, 
      author={Ben Adcock and Michael Griebel and Gregor Maier},
      year={2025},
      eprint={2410.23440},
      archivePrefix={arXiv},
      primaryClass={cs.LG},
      url={https://arxiv.org/abs/2410.23440}, 
}

@article{mccabe2023stability,
    title={Towards Stability of Autoregressive Neural Operators},
    author={Michael McCabe and Peter Harrington and Shashank Subramanian and Jed Brown},
    journal={Transactions on Machine Learning Research},
    issn={2835-8856},
    year={2023},
    url={https://openreview.net/forum?id=RFfUUtKYOG},
    note={}
}

@article{kim2025operatorlearning,
    title = {Neural operators learn the local physics of magnetohydrodynamics},
    journal = {Computers \& Fluids},
    volume = {297},
    pages = {106661},
    year = {2025},
    issn = {0045-7930},
    doi = {https://doi.org/10.1016/j.compfluid.2025.106661},
    url = {https://www.sciencedirect.com/science/article/pii/S0045793025001215},
    author = {Taeyoung Kim and Youngsoo Ha and Myungjoo Kang}
}

@article{makwana2020properties,
  title = {Properties of Magnetohydrodynamic Modes in Compressively Driven Plasma Turbulence},
  author = {Makwana, K. D. and Yan, Huirong},
  journal = {Phys. Rev. X},
  volume = {10},
  issue = {3},
  pages = {031021},
  numpages = {15},
  year = {2020},
  month = {Jul},
  publisher = {American Physical Society},
  doi = {10.1103/PhysRevX.10.031021},
  url = {https://link.aps.org/doi/10.1103/PhysRevX.10.031021}
}

@inproceedings{shikhman2026one,
    title={One Operator to Rule Them All? On Boundary-Indexed Operator Families in Neural {PDE} Solvers},
    author={Lennon Shikhman},
    booktitle={AI{\&}PDE: ICLR 2026 Workshop on AI and Partial Differential Equations},
    year={2026},
    url={https://openreview.net/forum?id=lDjWQ9UxRy}
}

\end{document}